\documentclass[3p,review]{elsarticle}

\usepackage{amssymb}
\usepackage{amsbsy}
\usepackage{amsmath}
\usepackage{amsfonts}

 \usepackage{amsthm}

\usepackage{lineno}
\usepackage{tabulary}
\biboptions{comma,square,sort&compress}

\usepackage{color}

\journal{Theoretical Computer Science}

\begin{document}

\begin{frontmatter}



\title{Weighted Last-Step Min-Max Algorithm with Improved Sub-Logarithmic Regret}

\author[T]{Edward Moroshko}
\ead{edward.moroshko@gmail.com}
\author[T]{Koby Crammer}
\ead{koby@ee.technion.ac.il}
\address[T]{Department of Electrical Engineering, Technion, Israel}

\begin{abstract}
In online learning the performance of an algorithm is typically compared to the performance of a fixed function from some class, with a quantity called regret. Forster~\cite{Forster} proposed a last-step min-max algorithm which was somewhat simpler than the algorithm of Vovk~\cite{vovkAS}, yet with the same regret. In fact the algorithm he analyzed assumed that the choices of the adversary are bounded, yielding artificially only the two extreme cases.  We fix this problem by weighing the examples in such a way that the min-max problem will be well defined, and provide analysis with logarithmic regret that may have better multiplicative factor than both bounds of Forster~\cite{Forster} and Vovk~\cite{vovkAS}.  We also derive a new bound that may be sub-logarithmic, as a recent bound of Orabona et.al~\cite{OrabonaCBG12}, but may have better multiplicative factor. Finally, we analyze the algorithm in a weak-type of non-stationary setting, and show a bound that is sublinear if the non-stationarity is sub-linear as well.

\end{abstract}


\end{frontmatter}

% \linenumbers

%
%  Macros for Thesis
%%%%%%%%%%%%%%%%%%%%%%%%%%%%%%%%%%%%%%%%%%%%%%%%%%%%%%%%%%%%%%%%%%%%%%%%%%%%%

\newtheorem{theorem}{Theorem}
\newtheorem{lemma}[theorem]{Lemma}
\newtheorem{corollary}[theorem]{Corollary}

\newtheorem{Remark}{Remark}
\def\proof{\par\penalty-1000\vskip .5 pt\noindent{\bf Proof\/: }}
\def\proofsketch{\par\penalty-1000\vskip .5 pt\noindent{\bf Proof sketch\/: }}
\def\ProofSketch{\par\penalty-1000\vskip .1 pt\noindent{\bf Proof sketch\/: }}
\newcommand{\QED}{\hfill$\;\;\;\rule[0.1mm]{2mm}{2mm}$\\}

\newcommand{\todo}[1]{{~\\\bf TODO: {#1}}~\\}

%%%%%%%%%%%%%%%%%%%%%%%%%%%%%%%%%%%%%%%%%%%%%%%%%%%%%%%%%%%%%%%%%%
% General
%%%%%%%%%%%%%%%%%%%%%%%%%%%%%%%%%%%%%%%%%%%%%%%%%%%%%%%%%%%%%%%%%%
\newfont{\msym}{msbm10}
\newcommand{\reals}{\mathbb{R}}%Re}%{mbox{\msym R}}
\newcommand{\half}{\frac{1}{2}}
\newcommand{\sign}{{\rm sign}}
\newcommand{\paren}[1]{\left({#1}\right)}
\newcommand{\brackets}[1]{\left[{#1}\right]}
\newcommand{\braces}[1]{\left\{{#1}\right\}}
\newcommand{\ceiling}[1]{\left\lceil{#1}\right\rceil}
\newcommand{\abs}[1]{\left\vert{#1}\right\vert}
\newcommand{\tr}{{\rm Tr}}
\newcommand{\pr}[1]{{\rm Pr}\left[{#1}\right]}
\newcommand{\prp}[2]{{\rm Pr}_{#1}\left[{#2}\right]}
\newcommand{\Exp}[1]{{\rm E}\left[{#1}\right]}
\newcommand{\Expp}[2]{{\rm E}_{#1}\left[{#2}\right]}
\newcommand{\eqdef}{\stackrel{\rm def}{=}}
\newcommand{\comdots}{, \ldots ,}
\newcommand{\true}{\texttt{True}}
\newcommand{\false}{\texttt{False}}
\newcommand{\mcal}[1]{{\mathcal{#1}}}
\newcommand{\argmin}[1]{\underset{#1}{\mathrm{argmin}} \:}
\newcommand{\normt}[1]{\left\Vert {#1} \right\Vert^2}
\newcommand{\step}[1]{\left[#1\right]_+}
\newcommand{\1}[1]{[\![{#1}]\!]}
\newcommand{\diag}{{\textrm{diag}}}
\newcommand{\KL}{{\textrm{D}_{\textrm{KL}}}}
\newcommand{\IS}{{\textrm{D}_{\textrm{IS}}}}
\newcommand{\EU}{{\textrm{D}_{\textrm{EU}}}}

%%%%%%%%%%%%%%%%%%%%%%%%%%%%%%%%%%%%%%%%%%%%%%%%%%%%%%%%%%
% Control symbols
%%%%%%%%%%%%%%%%%%%%%%%%%%%%%%%%%%%%%%%%%%%%%%%%%%%%%%%%%%
\newcommand{\leftmarginpar}[1]{\marginpar[#1]{}}
\newcommand{\figline}{\rule{0.50\textwidth}{0.5pt}}
\newcommand{\pseudocodefont}{\normalsize}
\newcommand{\nolineskips}{
\setlength{\parskip}{0pt}
\setlength{\parsep}{0pt}
\setlength{\topsep}{0pt}
\setlength{\partopsep}{0pt}
\setlength{\itemsep}{0pt}}

%%%%%%%%%%%%%%%%%%%%%%%%%%%%%%%%%%%%%%%%%%%%%%%%%%%%%%%%%%%
% Equations and references
%%%%%%%%%%%%%%%%%%%%%%%%%%%%%%%%%%%%%%%%%%%%%%%%%%%%%%%%%%%
\newcommand{\beq}[1]{\begin{equation}\label{#1}}
\newcommand{\eeq}{\end{equation}}
\newcommand{\beqa}{\begin{eqnarray}}
\newcommand{\eeqa}{\end{eqnarray}}
\newcommand{\secref}[1]{Sec.~\ref{#1}}
\newcommand{\figref}[1]{Fig.~\ref{#1}}
\newcommand{\exmref}[1]{Example~\ref{#1}}
\newcommand{\thmref}[1]{Theorem~\ref{#1}}
\newcommand{\sthmref}[1]{Thm.~\ref{#1}}
\newcommand{\defref}[1]{Definition~\ref{#1}}
\newcommand{\remref}[1]{Remark~\ref{#1}}
\newcommand{\chapref}[1]{Chapter~\ref{#1}}
\newcommand{\appref}[1]{Appendix~\ref{#1}}
\newcommand{\lemref}[1]{Lemma~\ref{#1}}
\newcommand{\propref}[1]{Proposition~\ref{#1}}
\newcommand{\claimref}[1]{Claim~\ref{#1}}

\renewcommand{\corref}[1]{Corollary~\ref{#1}}
\newcommand{\scorref}[1]{Cor.~\ref{#1}}
\newcommand{\tabref}[1]{Table~\ref{#1}}
\newcommand{\tran}[1]{{#1}^{\top}}
\newcommand{\norm}{\mcal{N}}
\newcommand{\eqsref}[1]{Eqns.~(\ref{#1})}

%%%%%%%%%%%%%%%%%%%%%%%%%%%%%%%%%%%%%%%%%%%%%%%%%%%%%%%%%%%
% bold, up, down
%%%%%%%%%%%%%%%%%%%%%%%%%%%%%%%%%%%%%%%%%%%%%%%%%%%%%%%%%%%
\newcommand{\mb}[1]{{\boldsymbol{#1}}}
\newcommand{\up}[2]{{#1}^{#2}}
\newcommand{\dn}[2]{{#1}_{#2}}
\newcommand{\du}[3]{{#1}_{#2}^{#3}}
\newcommand{\textl}[2]{{$\textrm{#1}_{\textrm{#2}}$}}

%%%%%%%%%%%%%%%%%%%%%%%%%%%%%%%%%%%%%%%%%%%%%%%%%%%%%%%%%%%
% vectors \va
%%%%%%%%%%%%%%%%%%%%%%%%%%%%%%%%%%%%%%%%%%%%%%%%%%%%%%%%%%%
\newcommand{\vx}{\mathbf{x}}
\newcommand{\vxi}[1]{\vx_{#1}}
\newcommand{\vxii}{\vxi{t}}

\newcommand{\yi}[1]{y_{#1}}
\newcommand{\yii}{\yi{t}}
\newcommand{\hyi}[1]{\hat{y}_{#1}}
\newcommand{\hyii}{\hyi{i}}

\newcommand{\vy}{\mb{y}}
\newcommand{\vyi}[1]{\vy_{#1}}
\newcommand{\vyii}{\vyi{i}}

\newcommand{\vn}{\mb{\nu}}
\newcommand{\vni}[1]{\vn_{#1}}
\newcommand{\vnii}{\vni{i}}

\newcommand{\vmu}{\mb{\mu}}
\newcommand{\vmus}{{\vmu^*}}
\newcommand{\vmuts}{{\vmus}^{\top}}
\newcommand{\vmui}[1]{\vmu_{#1}}
\newcommand{\vmuii}{\vmui{i}}

\newcommand{\vmut}{\vmu^{\top}}
\newcommand{\vmuti}[1]{\vmut_{#1}}
\newcommand{\vmutii}{\vmuti{i}}

\newcommand{\vsigma}{\mb \sigma}
\newcommand{\msigma}{\Sigma}
\newcommand{\msigmas}{{\msigma^*}}
\newcommand{\msigmai}[1]{\msigma_{#1}}
\newcommand{\msigmaii}{\msigmai{t}}

\newcommand{\mups}{\Upsilon}
\newcommand{\mupss}{{\mups^*}}
\newcommand{\mupsi}[1]{\mups_{#1}}
\newcommand{\mupsii}{\mupsi{i}}
\newcommand{\upssl}{\upsilon^*_l}

\newcommand{\vu}{\mathbf{u}}
\newcommand{\vut}{\tran{\vu}}
\newcommand{\vui}[1]{\vu_{#1}}
\newcommand{\vuti}[1]{\vut_{#1}}
\newcommand{\hvu}{\hat{\vu}}
\newcommand{\hvut}{\tran{\hvu}}
\newcommand{\hvur}[1]{\hvu_{#1}}
\newcommand{\hvutr}[1]{\hvut_{#1}}
\newcommand{\vw}{\mb{w}}
\newcommand{\vwi}[1]{\vw_{#1}}
\newcommand{\vwii}{\vwi{t}}
\newcommand{\vwti}[1]{\vwt_{#1}}
\newcommand{\vwt}{\tran{\vw}}

\newcommand{\tvw}{\tilde{\mb{w}}}
\newcommand{\tvwi}[1]{\tvw_{#1}}
\newcommand{\tvwii}{\tvwi{t}}

\newcommand{\vh}{\mb{h}}

\newcommand{\vv}{\mb{v}}
\newcommand{\vvt}{\tran{\vv}}

\newcommand{\vvi}[1]{\vv_{#1}}
\newcommand{\vvti}[1]{\vvt_{#1}}
\newcommand{\lambdai}[1]{\lambda_{#1}}
\newcommand{\Lambdai}[1]{\Lambda_{#1}}

\newcommand{\vxt}{\tran{\vx}}
\newcommand{\hvx}{\hat{\vx}}
\newcommand{\hvxi}[1]{\hvx_{#1}}
\newcommand{\hvxii}{\hvxi{i}}
\newcommand{\hvxt}{\tran{\hvx}}
\newcommand{\hvxti}[1]{\hvxt_{#1}}
\newcommand{\hvxtii}{\hvxti{i}}
\newcommand{\vxti}[1]{\vxt_{#1}}
\newcommand{\vxtii}{\vxti{i}}

\newcommand{\vb}{\mb{b}}
\newcommand{\vbt}{\tran{\vb}}
\newcommand{\vbi}[1]{\vb_{#1}}

\newcommand{\hvy}{\hat{\vy}}
\newcommand{\hvyi}[1]{\hvy_{#1}}

%%%%%%%%%%%%%%%%%%%%%%%%%%%%%%%%%%%%%%%%%%%%%%%%%%%%%%%%%%%%%%%%%
% Matrices (\mA)
%%%%%%%%%%%%%%%%%%%%%%%%%%%%%%%%%%%%%%%%%%%%%%%%%%%%%%%%%%%%%%%%%

\renewcommand{\mp}{P}
\newcommand{\mpd}{\mp^{(d)}}
\newcommand{\mpt}{\mp^T}
\newcommand{\tmp}{\tilde{\mp}}
\newcommand{\mpi}[1]{\mp_{#1}}
\newcommand{\mpti}[1]{\mpt_{#1}}
\newcommand{\mptii}{\mpti{i}}
\newcommand{\mpii}{\mpi{i}}
\newcommand{\mps}{Q}
\newcommand{\mpsi}[1]{\mps_{#1}}
\newcommand{\mpsii}{\mpsi{i}}
\newcommand{\tmpt}{\tmp^T}
\newcommand{\mz}{Z}
\newcommand{\mv}{V}
\newcommand{\mvi}[1]{\mv_{#1}}
\newcommand{\mvt}{V^T}
\newcommand{\mvti}[1]{\mvt_{#1}}
\newcommand{\mzt}{\mz^T}
\newcommand{\tmz}{\tilde{\mz}}
\newcommand{\tmzt}{\tmz^T}
\newcommand{\mx}{\mathbf{X}}
\newcommand{\ma}{\mathbf{A}}
\newcommand{\mxs}[1]{\mx_{#1}}

\newcommand{\mai}[1]{\ma_{#1}}
\newcommand{\mat}{\tran{\ma}}
\newcommand{\mati}[1]{\mat_{#1}}

\newcommand{\mc}{{C}}
\newcommand{\mci}[1]{\mc_{#1}}
\newcommand{\mcti}[1]{\mct_{#1}}

\newcommand{\md}{{\mathbf{D}}}
\newcommand{\mdi}[1]{\md_{#1}}

\newcommand{\mxi}[1]{\textrm{diag}^2\paren{\vxi{#1}}}
\newcommand{\mxii}{\mxi{i}}

\newcommand{\hmx}{\hat{\mx}}
\newcommand{\hmxi}[1]{\hmx_{#1}}
\newcommand{\hmxii}{\hmxi{i}}
\newcommand{\hmxt}{\hmx^T}
\newcommand{\mxt}{\mx^\top}
\newcommand{\mi}{\mathbf{I}}
\newcommand{\mq}{Q}
\newcommand{\mqt}{\mq^T}
\newcommand{\mlam}{\Lambda}
%\newcommand{\ma}{A}
%\newcommand{\ms}{S}
%\newcommand{\mt}{T}

%%%%%%%%%%%%%%%%%%%%%%%%%%%%%%%%%%%%%%%%%%%%%%%%%%%%%%%%%%%
% mathcal
%%%%%%%%%%%%%%%%%%%%%%%%%%%%%%%%%%%%%%%%%%%%%%%%%%%%%%%%%%%
\renewcommand{\L}{\mcal{L}}
\newcommand{\R}{\mcal{R}}
\newcommand{\X}{\mcal{X}}
\newcommand{\Y}{\mcal{Y}}
\newcommand{\F}{\mcal{F}}
\newcommand{\nur}[1]{\nu_{#1}}
\newcommand{\lambdar}[1]{\lambda_{#1}}
\newcommand{\gammai}[1]{\gamma_{#1}}
\newcommand{\gammaii}{\gammai{i}}
\newcommand{\alphai}[1]{\alpha_{#1}}
\newcommand{\alphaii}{\alphai{i}}
\newcommand{\lossp}[1]{\ell_{#1}}
\newcommand{\eps}{\epsilon}
\newcommand{\epss}{\eps^*}
\newcommand{\lsep}{\lossp{\eps}}
\newcommand{\lseps}{\lossp{\epss}}
\newcommand{\T}{\mcal{T}}

%%%%%%%%%%%%%%%%%%%%%%%%%%%%%%%%%%%%%%%%%%%%%%%%%%%%%%%%%%%
% Notes
%%%%%%%%%%%%%%%%%%%%%%%%%%%%%%%%%%%%%%%%%%%%%%%%%%%%%%%%%%%
\newcommand{\kc}[1]{\begin{center}\fbox{\parbox{3in}{{\textcolor{green}{KC: #1}}}}\end{center}}
\newcommand{\edward}[1]{\begin{center}\fbox{\parbox{3in}{{\textcolor{red}{EM: #1}}}}\end{center}}
\newcommand{\nv}[1]{\begin{center}\fbox{\parbox{3in}{{\textcolor{blue}{NV: #1}}}}\end{center}}

\newcommand{\newstuffa}[2]{#2}
\newcommand{\newstufffroma}[1]{}
\newcommand{\newstufftoa}{}

\newcommand{\newstuff}[2]{#2}
\newcommand{\newstufffrom}[1]{}
\newcommand{\newstuffto}{}
\newcommand{\oldnote}[2]{}

\newcommand{\commentout}[1]{}
\newcommand{\mypar}[1]{\medskip\noindent{\bf #1}}

%%%%%%%%%%%%%%%%%%%%%%%%%%%%%%%%%%%%%%%%%%%%%%%%%%%%%%%%%%%
% other
%%%%%%%%%%%%%%%%%%%%%%%%%%%%%%%%%%%%%%%%%%%%%%%%%%%%%%%%%%%
% inner products
\newcommand{\inner}[2]{\left< {#1} , {#2} \right>}
\newcommand{\kernel}[2]{K\left({#1},{#2} \right)}
\newcommand{\tprr}{\tilde{p}_{rr}}
\newcommand{\hxr}{\hat{x}_{r}}
\newcommand{\projalg}{{PST }}%{\tt Projection }}
\newcommand{\projealg}[1]{$\textrm{PST}_{#1}~$}%{\tt Projection }}
\newcommand{\gradalg}{{GST }}%\tt Gradient }}

\newcounter {mySubCounter}
\newcommand {\twocoleqn}[4]{
  \setcounter {mySubCounter}{0} %
  \let\OldTheEquation \theequation %
  \renewcommand {\theequation }{\OldTheEquation \alph {mySubCounter}}%
  \noindent \hfill%
  \begin{minipage}{.40\textwidth}
\vspace{-0.6cm}
    \begin{equation}\refstepcounter{mySubCounter}
      #1
    \end {equation}
  \end {minipage}
~~~~~~
%\hfill %
  \addtocounter {equation}{ -1}%
  \begin{minipage}{.40\textwidth}
\vspace{-0.6cm}
    \begin{equation}\refstepcounter{mySubCounter}
      #3
    \end{equation}
  \end{minipage}%
  \let\theequation\OldTheEquation
}

\newcommand{\vzero}{\mb{0}}

\newcommand{\smargin}{\mcal{M}}

\newcommand{\ai}[1]{A_{#1}}
\newcommand{\bi}[1]{B_{#1}}
\newcommand{\aii}{\ai{i}}
\newcommand{\bii}{\bi{i}}
\newcommand{\betai}[1]{\beta_{#1}}
\newcommand{\betaii}{\betai{i}}
\newcommand{\mar}{M}
\newcommand{\mari}[1]{\mar_{#1}}
\newcommand{\marii}{\mari{i}}
\newcommand{\nmari}[1]{m_{#1}}
\newcommand{\nmarii}{\nmari{i}}

\newcommand{\erf}{\Phi}

\newcommand{\var}{V}
\newcommand{\vari}[1]{\var_{#1}}
\newcommand{\varii}{\vari{i}}

\newcommand{\varb}{v}
\newcommand{\varbi}[1]{\varb_{#1}}
\newcommand{\varbii}{\varbi{i}}

\newcommand{\vara}{u}
\newcommand{\varai}[1]{\vara_{#1}}
\newcommand{\varaii}{\varai{i}}

\newcommand{\marb}{m}
\newcommand{\marbi}[1]{\marb_{#1}}
\newcommand{\marbii}{\marbi{i}}

\newcommand{\algname}{{AROW}}
\newcommand{\rlsname}{{RLS}}
\newcommand{\mrlsname}{{MRLS}}

\newcommand{\phia}{\psi}
\newcommand{\phib}{\xi}

\newcommand{\amsigmaii}{\tilde{\msigma}_t}
\newcommand{\amsigmai}[1]{\tilde{\msigma}_{#1}}
\newcommand{\avmuii}{\tilde{\vmu}_i}
\newcommand{\avmui}[1]{\tilde{\vmu}_{#1}}
\newcommand{\amarbii}{\tilde{\marb}_i}
\newcommand{\avarbii}{\tilde{\varb}_i}
\newcommand{\avaraii}{\tilde{\vara}_i}
\newcommand{\aalphaii}{\tilde{\alpha}_i}

\newcommand{\svar}{v}
\newcommand{\smar}{m}
\newcommand{\nsmar}{\bar{m}}

\newcommand{\vnu}{\mb{\nu}}
\newcommand{\vnut}{\vnu^\top}
\newcommand{\vz}{\mb{z}}
\newcommand{\vZ}{\mb{Z}}
\newcommand{\fphi}{f_{\phi}}
\newcommand{\gphi}{g_{\phi}}

%%% Local Variables:
%%% mode: latex
%%% TeX-master: "nips2007"
%%% End:

\newcommand{\vtmui}[1]{\tilde{\vmu}_{#1}}
\newcommand{\vtmuii}{\vtmui{i}}

\newcommand{\zetai}[1]{\zeta_{#1}}
\newcommand{\zetaii}{\zetai{i}}

%%%%%%

\newcommand{\vstate}{\bf{s}}
\newcommand{\vstatet}[1]{\vstate_{#1}}
\newcommand{\vstatett}{\vstatet{t}}

\newcommand{\mtran}{\bf{\Phi}}
\newcommand{\mtrant}[1]{\mtran_{#1}}
\newcommand{\mtrantt}{\mtrant{t}}

\newcommand{\vstatenoise}{\bf{\eta}}
\newcommand{\vstatenoiset}[1]{\vstatenoise_{#1}}
\newcommand{\vstatenoisett}{\vstatenoiset{t}}

\newcommand{\vobser}{\bf{o}}
\newcommand{\vobsert}[1]{\vobser_{#1}}
\newcommand{\vobsertt}{\vobsert{t}}

\newcommand{\mobser}{\bf{H}}
\newcommand{\mobsert}[1]{\mobser_{#1}}
\newcommand{\mobsertt}{\mobsert{t}}

\newcommand{\vobsernoise}{\bf{\nu}}
\newcommand{\vobsernoiset}[1]{\vobsernoise_{#1}}
\newcommand{\vobsernoisett}{\vobsernoiset{t}}

\newcommand{\mstatenoisecov}{\bf{Q}}
\newcommand{\mstatenoisecovt}[1]{\mstatenoisecov_{#1}}
\newcommand{\mstatenoisecovtt}{\mstatenoisecovt{t}}

\newcommand{\mobsernoisecov}{\bf{R}}
\newcommand{\mobsernoisecovt}[1]{\mobsernoisecov_{#1}}
\newcommand{\mobsernoisecovtt}{\mobsernoisecovt{t}}

\newcommand{\vestate}{\bf{\hat{s}}}
\newcommand{\vestatet}[1]{\vestate_{#1}}
\newcommand{\vestatett}{\vestatet{t}}
\newcommand{\vestatept}[1]{\vestatet{#1}^+}
\newcommand{\vestatent}[1]{\vestatet{#1}^-}

\newcommand{\mcovar}{\bf{P}}
\newcommand{\mcovart}[1]{\mcovar_{#1}}
\newcommand{\mcovarpt}[1]{\mcovart{#1}^+}
\newcommand{\mcovarnt}[1]{\mcovart{#1}^-}

\newcommand{\mkalmangain}{\bf{K}}
\newcommand{\mkalmangaint}[1]{\mkalmangain_{#1}}

\newcommand{\vkalmangain}{\bf{\kappa}}
\newcommand{\vkalmangaint}[1]{\vkalmangain_{#1}}

\newcommand{\obsernoise}{{\nu}}
\newcommand{\obsernoiset}[1]{\obsernoise_{#1}}
\newcommand{\obsernoisett}{\obsernoiset{t}}

\newcommand{\obsernoisecov}{r}
\newcommand{\obsernoisecovt}[1]{\obsernoisecov_{#1}}
\newcommand{\obsernoisecovtt}{\obsernoisecov}%t{t}}

\newcommand{\obsnscv}{s}
\newcommand{\obsnscvt}[1]{\obsnscv_{#1}}
\newcommand{\obsnscvtt}{\obsnscvt{t}}

\newcommand{\Psit}[1]{\Psi_{#1}}
\newcommand{\Psitt}{\Psit{t}}

\newcommand{\Omegat}[1]{\Omega_{#1}}
\newcommand{\Omegatt}{\Omegat{t}}

\newcommand{\ellt}[1]{\ell_{#1}}
\newcommand{\gllt}[1]{g_{#1}}

\newcommand{\chit}[1]{\chi_{#1}}

\newcommand{\ms}{\mathcal{M}}
\newcommand{\us}{\mathcal{U}}
\newcommand{\as}{\mathcal{A}}

\newcommand{\mn}{M}
\newcommand{\un}{U}

\newcommand{\seti}[1]{S_{#1}}

\newcommand{\obj}{\mcal{C}}

\newcommand{\dta}[3]{d_{#3}\paren{#1,#2}}

\newcommand{\coa}{a}
\newcommand{\coc}{c}
\newcommand{\cob}{b}
\newcommand{\cor}{r}
\newcommand{\conu}{\nu}

\newcommand{\coat}[1]{\coa_{#1}}
\newcommand{\coct}[1]{\coc_{#1}}
\newcommand{\cobt}[1]{\cob_{#1}}
\newcommand{\cort}[1]{\cor_{#1}}
\newcommand{\conut}[1]{\conu_{#1}}

\newcommand{\coatt}{\coat{t}}
\newcommand{\coctt}{\coct{t}}
\newcommand{\cobtt}{\cobt{t}}
\newcommand{\cortt}{\cort{t}}
\newcommand{\conutt}{\conut{t}}

\newcommand{\rb}{R_B}
\newcommand{\proj}{\textrm{proj}}

%\kc{add ``,'' and ``.'' in equations as needed}
%\edward{Done.}
\section{Introduction}
%\kc{vector $b$ vs scalar $b$}
We consider the online learning regression problem, in which a learning algorithm tries to predict real numbers in a sequence of rounds given some side-information or inputs $\vxi{t}\in\reals^d$. Real-world example applications for these algorithms are weather or stockmarket predictions. The goal of the algorithm is to have a small discrepancy between its predictions and the associated outcomes $\yi{t}\in\reals$. This discrepancy is measured with a loss function, such as the square loss. It is common to evaluate algorithms by their regret, the difference between the cumulative loss of an algorithm with the cumulative loss of any function taken from some class.

Forster~\cite{Forster} proposed a last-step min-max algorithm for online regression that makes a prediction assuming it is the last example to be observed, and the goal of the algorithm is indeed to minimize the regret with respect to linear functions. The resulting optimization problem he obtained was convex in both choice of the algorithm and the choice of the adversary, yielding an unbounded optimization problem. Forster circumvented this problem by assuming a bound $Y$ over the choices of the adversary that should be known to the algorithm, yet his analysis is for the version with no bound.

We propose a modified last-step min-max algorithm with weights over examples, that are controlled in a way to obtain a problem that is concave over the choices of the adversary and convex over the   choices of the algorithm. We analyze our algorithm and show a logarithmic-regret that may have a better multiplicative factor than the analysis of Forster. We derive additional analysis that is logarithmic in the loss of the reference function, rather than the number of rounds $T$. This behaviour was recently given by Orabona et.al~\cite{OrabonaCBG12} for a certain online-gradient decent algorithm. Yet, their bound~\cite{OrabonaCBG12} has a similar multiplicative factor to that of 
Forster~\cite{Forster}, while our bound has a potentially better multiplicative factor and it has the same dependency in the cumulative loss of the reference function as Orabona et.al~\cite{OrabonaCBG12}. Additionally, our algorithm and analysis are totally free of assuming the bound $Y$ or knowing its value.

Competing with the best {\em single} function might not suffice for some problems. In many real-world applications, the true target function is not fixed, but may change from time to time. We bound the performance of our algorithm also in non-stationary environment, where we measure the complexity of the non-stationary environment by the total deviation of a collection of linear functions from some fixed reference point. We show that our algorithm maintains an average loss close to that of the best sequence of functions, as long as the total of this deviation is sublinear in the number of rounds $T$.

%\kc{please add whats new from ALT}
%\edward{Done.}
A short version appeared in The 23rd International Conference
on
Algorithmic Learning Theory (ALT 2012). 
This journal version of the paper includes additionally: (1) Recursive form of the algorithm and comparison to other algorithms of the same form (\secref{Recursive_form}). (2) Kernel version of the algorithm (\secref{kernel_form}).
(3) MAP interpretation of the minimization problems (\remref{MAP1} and \remref{MAP2}). (4) All proofs and extended related-work section.

\section{Problem Setting}
\label{sec:problem_setting} 
We work in the online setting for regression evaluated with the
squared loss. Online algorithms work in rounds or iterations. On each
iteration an online algorithm receives an instance
$\vxi{t}\in\reals^d$ and predicts a real value $\hyi{t}\in\reals$, it
then receives a label $\yi{t}\in\reals$, possibly chosen by an
adversary, suffers loss $\ell_t(\textrm{alg})=\ell\paren{ \yi{t},
  \hyi{t} } = \paren{\hyi{t}- \yi{t} }^2$, updates its prediction
rule, and proceeds to the next round. The cumulative loss suffered by
the algorithm over $T$ iterations is, 
\begin{equation}
L_{T}(\textrm{alg})=\sum_{t=1}^{T}\ell_{t}(\textrm{alg})
\label{LT} ~.
\end{equation}
The goal of the algorithm is to perform well compared to any predictor
from some function class. 

A common choice is to compare the performance of an algorithm
with respect to {\em a single} function, or specifically a single
linear function,  $f(\vx)=\vxt\vu$, parameterized by a vector
$\vu\in\reals^d$. 
Denote by $\ell_t(\vu) = \paren{\vxti{t}\vu-\yi{t}}^2$ the instantaneous
loss of a vector $\vu$, and by $L_T(\vu) = \sum_t^T
\ell_t(\vu)$.
The regret with respect to $\vu$ is defined to be,
% We consider two
% alternatives. First, a common choice it to compete against any linear
% function $f(\vx)=\vxt\vu$, parametrized by some weight-vector $\vu\in\reals^d$. The regret is,
\[
R_T(\vu) = \sum_t^T (\yi{t}- \hyi{t})^2
%\paren{\vxti{t}\vwi{t-1}-\yi{t}}^2 
-  L_T(\vu) ~.
\]
A desired goal of the algorithm is to have ${R}_T(\vu) = o(T)$, that is, the
average loss suffered by the algorithm will converge to the average
loss of the best linear function $\vu$.  

Below in \secref{sec:non_stat} we will also consider an extension of
this form of regret, and evaluate the performance of an algorithm
against some $T$-tuple of functions, $(\vui{1} \comdots \vui{T})\in
\reals^d \times \dots \times \reals^d$,
\[
R_T(\vui{1}\comdots \vui{T})  = \sum_t^T (\yi{t}- \hyi{t})^2-  L_T(\vui{1} \comdots \vui{T})~, %(\yi{t}- \hyi{t})^2 - \sum_t^T \paren{\vxti{t}\vui{t}-\yi{t}}^2 ~.
\]
where $ L_T(\vui{1} \comdots \vui{T}) = \sum_t^T
\ell_t(\vui{t})$.
% A second choice we consider is an extension of the first one, in
% which, the performance of the algorithm is compared with the
% performance of a sequence of functions $(\vwi{1} \comdots \vwi{T})\in
% \reals^d \times \dots \times \reals^d$, that is,
% \[
% \mcal{R}_T = \sum_t^T (\yi{t}- \hyi{t})^2 - \inf_{\vui{1}\comdots\vui{T}}
% \sum_t^T \paren{\vxti{t}\vui{t}-\yi{t}}^2 ~.
% \]
Clearly, with no restriction of the $T$-tuple, any algorithm may
suffer a regret linear in $T$, as one can set $\vui{t} = \vxi{t}
(\yi{t}/\normt{\vxi{t}})$, and suffer zero quadratic loss in all rounds. Thus, we restrict below
the possible choices of $T$-tuple either explicitly, or implicitly via
some penalty.
%We thus restrict below
%the sequence of weight-vectors $(\vui{1} \comdots \vwi{T})$ to be
%close to each other in some sense.

%  algorithm is compared only with sequences of functions with
% total deviation that is sub-linear in $T$, that is, for which $\sum_t
% \Vert \vui{t-1} -\vui{t}\Vert = o(T)$. 
% %\nv{maybe it should be $\sum_t
% %\Vert \vui{t-1} -\vui{t}\Vert \leq o(T)$, because you say that it is sub linear}
% Clearly, if the total deviation
% is $\Omega(T)$ we can not expect a learning algorithm to achieve a
% vanishing averaged regret. Yet, it may be the case that the
% sequence of function is not converging $\lim_{t\rightarrow\infty} \sum_t
% \Vert \vui{t-1} -\vui{t}\Vert$ yet the algorithm will have vanishing
% average regret.

% A linear expert $\mathbf{u}\in\mathbb{R}^{d}$ makes the prediction
% $\mathbf{u}^{\top}\mathbf{x}$ on instance $\mathbf{x}\in\mathbb{R}^{d}$
% and its loss at each round $t$ is 
% \[
% \ell_{t}(\mathbf{u})=\left(y_{t}-\mathbf{u}^{\top}\mathbf{x}_{t}\right)^{2}
% \]
% After $T$ rounds, the total expert loss is 

% \[
% L_{T}(\mathbf{u})=\sum_{t=1}^{T}\ell_{t}(\mathbf{u})
% \]

% kkk

% Formally, at round $T$ the algorithm outputs,
% \[
% \hyi{T} = \arg\min_{\hyi{T}} \max_{\yi{T}} \brackets{\sum_{t=1}^{T} (\yi{t} -
%   \hyi{t})^2  - \inf_\vu \paren{\sum_{t=1}^T (\vut\vxi{t}-\yi{t})^2}}~.
% \]
% Note, that $\yi{T}$ and $\hyi{T}$ serves both as quantifiers (over the
% $\min$ and $\max$ operators, respectively), and as the optimal values
% over this optimization problem. 

% kkk

\section{A Last Step Min-Max Algorithm}
Our algorithm is derived based on a last-step min-max prediction,
proposed by Forster~\cite{Forster} and Takimoto and
Warmuth~\cite{TakimotoW00}. See also the work of Azoury and
Warmuth~\cite{AzouryWa01}. An algorithm following this approach outputs the
min-max prediction assuming the current iteration is the last one. The
algorithm we describe below is based on an extension of this
notion. For this purpose we introduce a weighted cumulative loss using
positive input-dependent weights $\left\{ a_{t}\right\} _{t=1}^{T}$,
\[
L_{T}^{\boldsymbol{a}}(\mathbf{u})=\sum_{t=1}^{T}a_{t}\left(y_{t}-\mathbf{u}^{\top}\mathbf{x}_{t}\right)^{2}
\quad,\quad
L_{T}^{\boldsymbol{a}}(\vui{1} \comdots
\vui{T})=\sum_{t=1}^{T}a_{t}\left(y_{t}-\vuti{t} \vxi{t}\right)^{2}~.
\]
The exact values of the weights $a_t$ will be defined below.

Our variant of the last step min-max algorithm predicts\footnote{$\yi{T}$ and $\hyi{T}$ serves both as quantifiers (over the
 $\min$ and $\max$ operators, respectively), and as the optimal values
 over this optimization problem. }
\begin{align}
\hyi{T} = \arg\min_{\hyi{T}} \max_{\yi{T}} \brackets{\sum_{t=1}^{T} (\yi{t} -
  \hyi{t})^2  - \inf_{\vu} \paren{b\left\Vert \mathbf{u}\right\Vert
    ^{2}+L_{T}^{\boldsymbol{a}}(\vu)}}~,
\label{minmax_algorithm_1}
\end{align}
for some positive constant $b>0$.
We next compute the actual prediction based on the optimal last step min-max solution. We
start with additional notation,
\begin{align}
\mathbf{A}_{t}
&=b\mathbf{I}+\sum_{s=1}^{t}a_{s}\mathbf{x}_{s}\mathbf{x}_{s}^{\top}&&\in\mathbb{R}^{d\times
  d}\label{Adef}\\
\mathbf{b}_{t}&=\sum_{s=1}^{t}a_{s}y_{s}\mathbf{x}_{s}&&\in\mathbb{R}^{d}\label{bdef} ~.
\end{align}
The solution of the internal infimum over $\vu$ is summarized in the
following lemma.
\begin{figure}[!t!]
\includegraphics[width=0.33\textwidth]{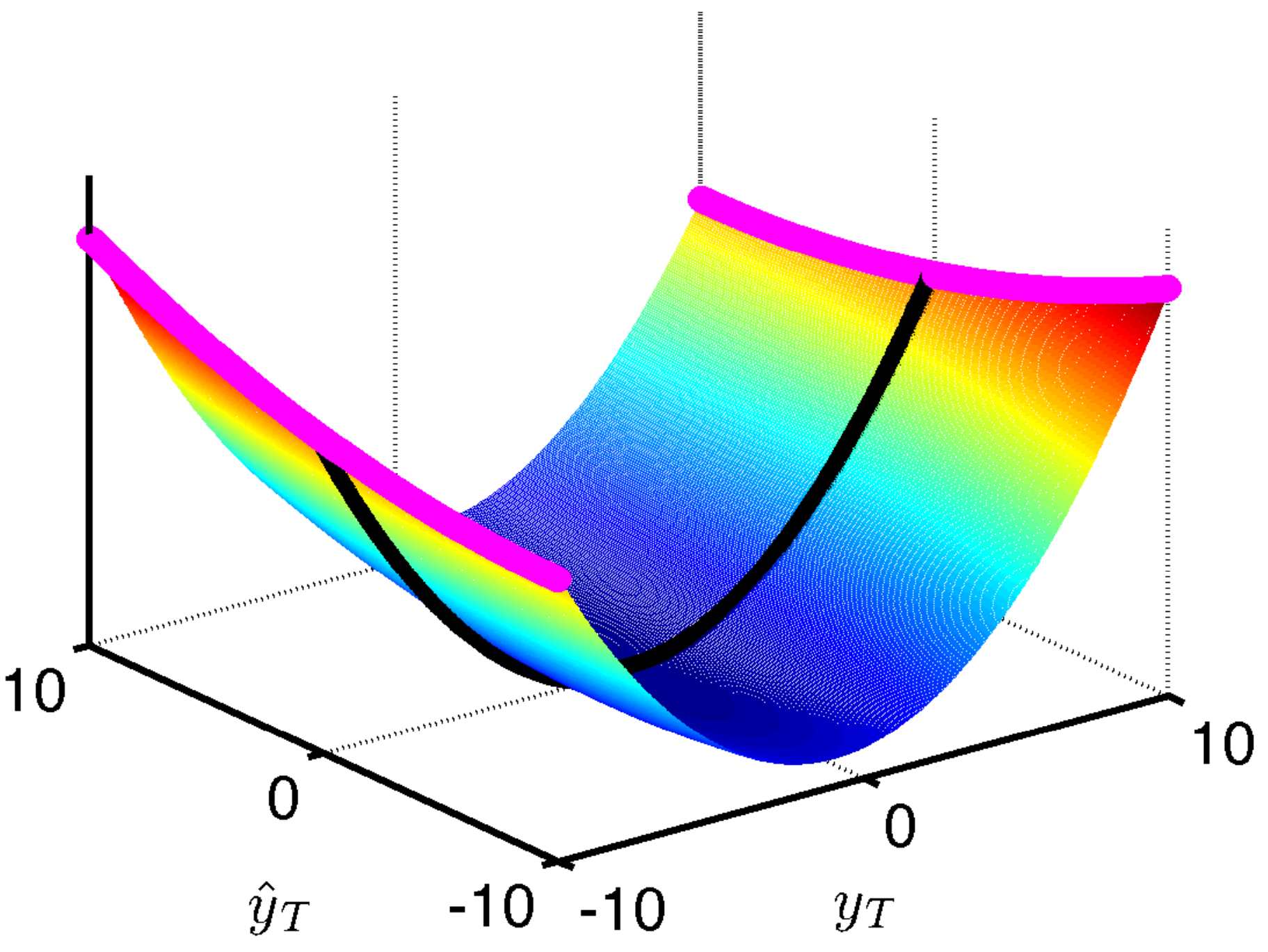}
\includegraphics[width=0.33\textwidth]{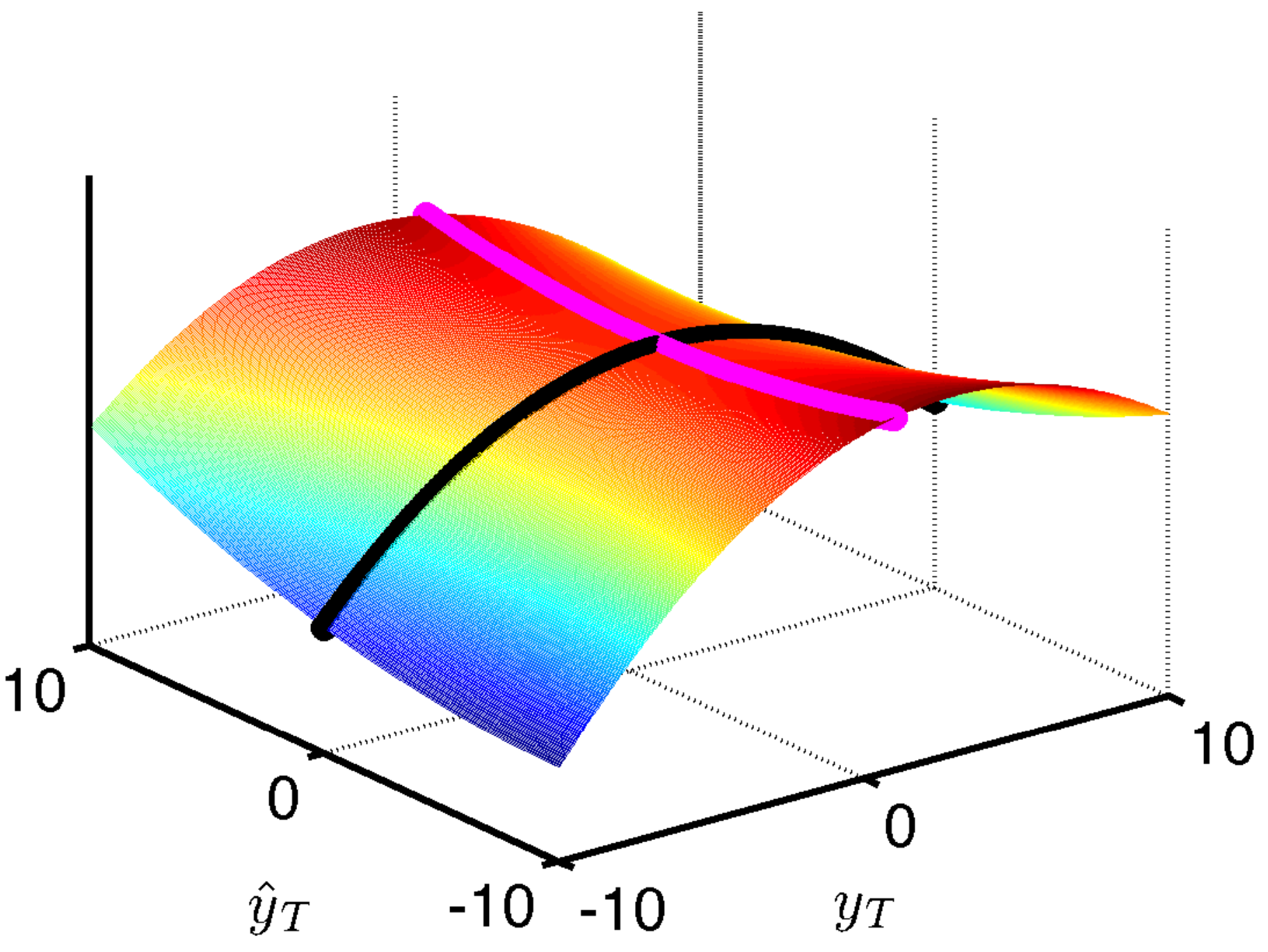}
\includegraphics[width=0.33\textwidth]{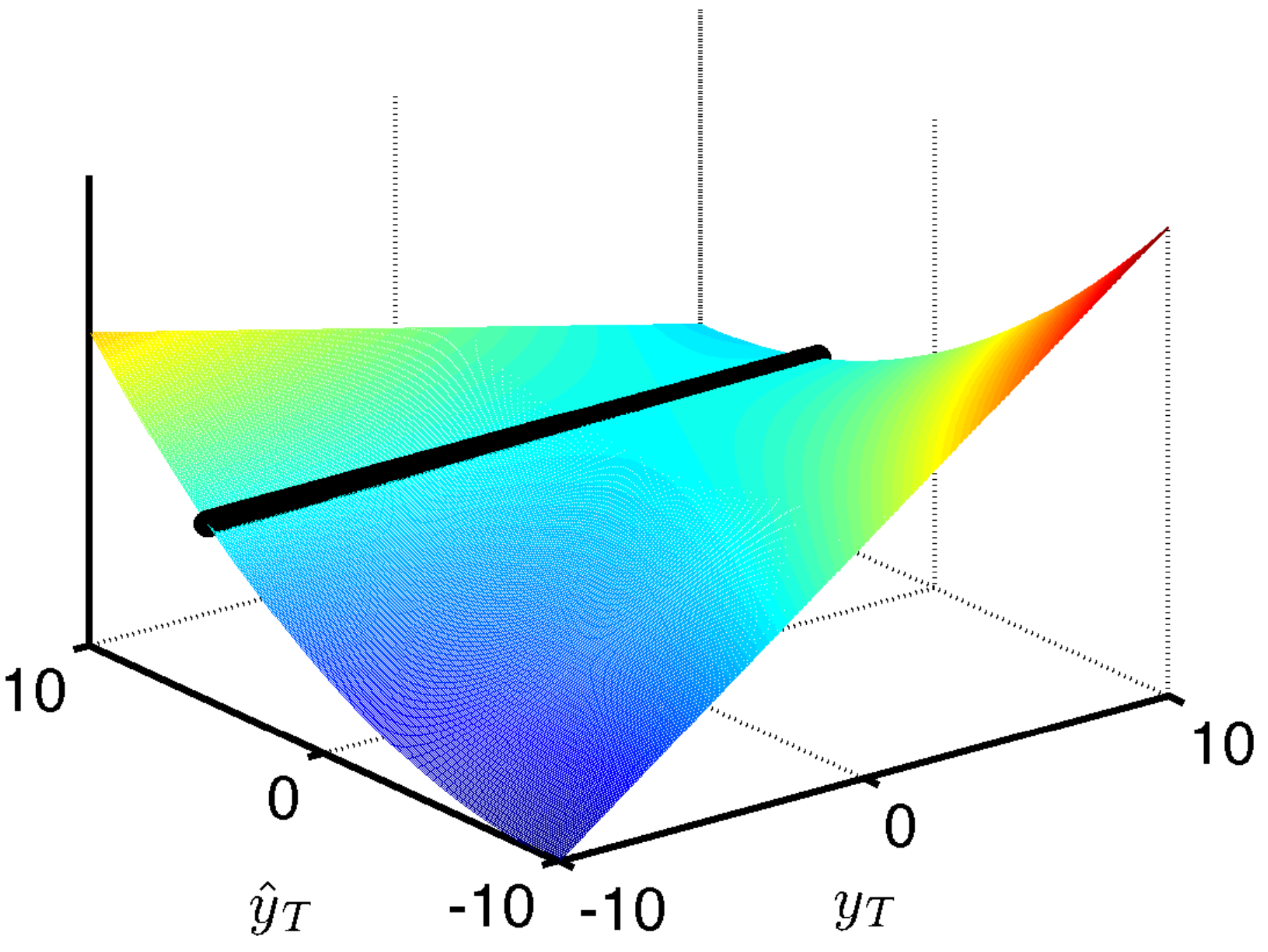}
\caption{
An illustration of the minmax objective function $G(\yi{T},\hyi{T})$
\eqref{minmax_objective}. The black line is the value of the objective
as a function of $\yi{T}$ for the optimal predictor $\hyi{T}$. 
Left: Forster's optimization function (convex in $\yi{T}$). 
Center: our optimization function (strictly  concave in $\yi{T}$,
case 1 in \thmref{thm:theorem1}).
Right: our optimization function (invariant to $\yi{T}$, case 2 in \thmref{thm:theorem1}).}
\label{fig:minmax}
\end{figure}
\begin{lemma}
\label{lem:lemma1}
For all $t\geq1$, the function $f\left(\mathbf{u}\right)=b\left\Vert \mathbf{u}\right\Vert ^{2}+\sum_{s=1}^{t}a_{s}\left(y_{s}-\mathbf{u}^{\top}\mathbf{x}_{s}\right)^{2}$
is minimal at a unique point $\mathbf{u}_{t}$ given by,
%  Furthermore, $\mathbf{u}_{t}$
% and $f(\mathbf{u}_{t})$ are given by 
\begin{align}
\mathbf{u}_{t}=\mathbf{A}_{t}^{-1}\mathbf{b}_{t}\quad\textrm{ and
}\quad
f(\mathbf{u}_{t})=\sum_{s=1}^{t}a_{s}y_{s}^{2}-\mathbf{b}_{t}^{\top}\mathbf{A}_{t}^{-1}\mathbf{b}_{t}
~. \label{optimal_solution}
\end{align}
\end{lemma}
%\kc{please add this proof}
%\edward{Done.}
%The proof is similar to the proof of Lemma 1 by Forster~\cite{Forster}.
\begin{proof}
From 
\begin{eqnarray*}
f\left(\mathbf{u}\right) & = & b\left\Vert \mathbf{u}\right\Vert ^{2}+\sum_{s=1}^{t}a_{s}\left(y_{s}-\mathbf{u}^{\top}\mathbf{x}_{s}\right)^{2}\\
 & = & \sum_{s=1}^{t}a_{s}y_{s}^{2}-2\sum_{s=1}^{t}\mathbf{u}^{\top}\left(a_{s}y_{s}\mathbf{x}_{s}\right)+\mathbf{u}^{\top}\left(b\mathbf{I}+\sum_{s=1}^{t}a_{s}\mathbf{x}_{s}\mathbf{x}_{s}^{\top}\right)\mathbf{u}\\
 & \underset{=}{\eqref{Adef},}\eqref{bdef} & \sum_{s=1}^{t}a_{s}y_{s}^{2}-2\mathbf{u}^{\top}\mathbf{b}_{t}+\mathbf{u}^{\top}\mathbf{A}_{t}\mathbf{u}
\end{eqnarray*}
it follows that $\nabla
f\left(\mathbf{u}\right)=2\mathbf{A}_{t}\mathbf{u}-2\mathbf{b}_{t},\:
\triangle f(\mathbf{u})=2\mathbf{A}_{t}$.
Thus $f$ is convex and it is minimal if $\nabla f\left(\mathbf{u}\right)=0$,
i.e. for $\mathbf{u}=\mathbf{A}_{t}^{-1}\mathbf{b}_{t}$. This show
that $\mathbf{u}_{t}=\mathbf{A}_{t}^{-1}\mathbf{b}_{t}$ and we obtain

\[
f\left(\mathbf{u}_{t}\right)=f\left(\mathbf{A}_{t}^{-1}\mathbf{b}_{t}\right)=\sum_{s=1}^{t}a_{s}y_{s}^{2}-2\mathbf{b}_{t}^{\top}\mathbf{A}_{t}^{-1}\mathbf{b}_{t}+\mathbf{b}_{t}^{\top}\mathbf{A}_{t}^{-1}\mathbf{A}_{t}\mathbf{A}_{t}^{-1}\mathbf{b}_{t}=\sum_{s=1}^{t}a_{s}y_{s}^{2}-\mathbf{b}_{t}^{\top}\mathbf{A}_{t}^{-1}\mathbf{b}_{t}~.
\]
\QED
\end{proof}
\begin{Remark}
\label{MAP1}
The minimization problem in \lemref{lem:lemma1} can be interpreted as MAP estimator
of $\mathbf{u}$ based on the sequence $\left\{ \left(\mathbf{x}_{s},y_{s}\right)\right\} _{s=1}^{t}$
in the following generative model:
\begin{eqnarray*}
\mathbf{u} & \sim & N\left(0,\sigma_{b}^{2}\mathbf{I}\right)\\
y_{s} & \sim & N\left(\mathbf{x}_{s}^{\top}\mathbf{u},\sigma_{s}^{2}\right)~,
\end{eqnarray*}
where $\sigma_{b}^{2}=\frac{1}{2b}$ and $\sigma_{s}^{2}=\frac{1}{2a_{s}}$.

Under the model we calculate,
\begin{eqnarray}
\mathbf{u}_{MAP} & = & \arg\max_{\mathbf{u}}P\left(\mathbf{u}\mid\left\{ \mathbf{x}_{s}\right\} ,\left\{ y_{s}\right\} \right) \nonumber \\
 & = & \arg\max_{\mathbf{u}}\left[P\left(\mathbf{u}\right)\prod_{s=1}^{t}P\left(y_{s}\mid\mathbf{u},\mathbf{x}_{s}\right)\right] \nonumber \\
 & = & \arg\min_{\mathbf{u}}\left[-\log P\left(\mathbf{u}\right)-\sum_{s=1}^{t}\log P\left(y_{s}\mid\mathbf{u},\mathbf{x}_{s}\right)\right]~. \label{u_map}
\end{eqnarray}
By our gaussian generative model,
\begin{align*}
&-\log P\left(\mathbf{u}\right)&=&\log\left(2\pi\sigma_{b}^{2}\right)^{d/2}+\frac{1}{2\sigma_{b}^{2}}\left\Vert \mathbf{u}\right\Vert ^{2}\\
&-\log P\left(y_{s}\mid\mathbf{u},\mathbf{x}_{s}\right)&=&\log\left(2\pi\sigma_{s}^{2}\right)^{1/2}+\frac{1}{2\sigma_{s}^{2}}\left(y_{s}-\mathbf{x}_{s}^{\top}\mathbf{u}\right)^{2}~.
\end{align*}
Substituting in \eqref{u_map} we get
\[
 \mathbf{u}_{MAP}=\arg\min_{\mathbf{u}}\left[\frac{1}{2\sigma_{b}^{2}}\left\Vert \mathbf{u}\right\Vert ^{2}+\sum_{s=1}^{t}\frac{1}{2\sigma_{s}^{2}}\left(y_{s}-\mathbf{x}_{s}^{\top}\mathbf{u}\right)^{2}\right]~,
\]
and by using $\frac{1}{2\sigma_{b}^{2}}=b$,
$\frac{1}{2\sigma_{s}^{2}}=a_{s}$ we get the minimization
problem of \lemref{lem:lemma1}. 
\end{Remark}

Substituting \eqref{optimal_solution} back in
\eqref{minmax_algorithm_1} we obtain the following form of the minmax
problem,
%
% Before proving the theorem we discuss three possible cases, one
% used by \cite{Forster} and two in the proof of the theorem.
% We show below that one can write the last step minmax problem as
\begin{align}
\min_{\hyi{T}} \max_{\yi{T}} G(\yi{T},\hyi{T}) \quad\textrm{ for }
\quad G(\yi{T},\hyi{T})= \alpha(a_T) \yi{T}^2 + 2 \beta(a_T,
\hyi{T}) \yi{T} + \hyi{T}^2 ~,\label{minmax_objective}
\end{align}
for some functions $\alpha(a_T)$ and $\beta(a_T, \hyi{T})$. Clearly,
for this problem to be well defined the function $G$ should be convex
in $\hyi{T}$ and concave in $\yi{T}$.

A previous choice, proposed by Forster~\cite{Forster}, is to have uniform weights
and set $a_t=1$ (for $t=1 \comdots T$), which for the particular function $\alpha(a_T)$
yields $\alpha(a_T)>0$. Thus, $G(\yi{T},\hyi{T})$ is a convex function
in $\yi{T}$, implying that the optimal value of $G$ is not bounded
from above. Forster~\cite{Forster} addressed this problem by restricting
$\yi{T}$ to belong to a predefined interval $[-Y,Y]$, known also to
the learner. As a consequence, the adversary optimal prediction is in
fact either $\yi{T}=Y$ or $\yi{T}=-Y$, which in turn yields an optimal
predictor which is clipped at this bound, \(
\hat{y}_{T}={\rm clip}\left(\mathbf{b}_{T-1}^{\top}\mathbf{A}_{T}^{-1}\mathbf{x}_{T},Y\right)\),
where for $y>0$ we define ${\rm clip}(x,y)=x$ if $\vert x \vert \leq y$ and
${\rm clip}(x,y) =y\, \sign(x)$, otherwise.

This phenomena is illustrated in the left panel
of \figref{fig:minmax} (best viewed in color).  For the minmax optimization function defined
by Forster~\cite{Forster}, fixing some value of $\hat{y}_{T}$, the function
is convex in $y_{T}$, and the adversary would achieve a maximal value
at the boundary of the feasible values of $y_{T}$ interval. That is,
either $y_{T}=Y$ or $y_{T}=-Y$, as indicated by the two magenta lines
at $\yi{T}=\pm10$. The optimal predictor $\hat{y}_{T}$ is
achieved somewhere along the lines $y_{T}=Y$ or $y_{T}=-Y$.

% and showed that it is optimal when the last outcome variable is bounded,
% that is $y_{T}\in[-Y,Y]$. We next give similar result, but we do
% not assume that $y_{T}$ is bounded. Instead, other assumption is
% required.

% The difference between our and Forster optimization function is illustrated
% in the next figure:

We propose an alternative approach to make the minmax optimal solution
bounded by appropriately setting the weight $a_T$ such that
$G(\yi{T},\hyi{T})$ is concave in $\yi{T}$ for a constant
$\hat{y}_{T}$. We explicitly consider two cases.  
First, set $a_T$ such that $G(\yi{T},\hyi{T})$ is {\em strictly
  concave} in $\yi{T}$, and thus attains a single maximum with no need
to artificially restrict the value of $\yi{T}$. In this case our
function is concave in $y_{T}$ in the first option and has a maximum
point, which is the worst adversary. The optimal predictor
$\hat{y}_{T}$ is achieved in the unique saddle point, as illustrated
in the center panel of \figref{fig:minmax}.
A second case is to set $a_T$ such that $\alpha(a_T)=0$ and the
minmax function $G(\yi{T},\hyi{T})$ becomes linear in $\yi{T}$. Here,
the optimal prediction is achieved by choosing $\hyi{T}$ such that
$\beta(a_T, \hyi{T})=0$ which turns $G(\yi{T},\hyi{T})$ to be
invariant to $\yi{T}$, as illustrated in the right panel of
\figref{fig:minmax}.  % the for a
% specific choice of $\hat{y}_{T}$, because for other $\hat{y}_{T}$ the
% optimization function can made large by choosing $y_{T}$ with large
% absolute value.

Equipped with
\lemref{lem:lemma1} we develop the optimal solution of the min-max
predictor, summarized in the following theorem.
\begin{theorem}
\label{thm:theorem1}
Assume that $1+a_{T}\mathbf{x}_{T}^{\top}\mathbf{A}_{T-1}^{-1}\mathbf{x}_{T}-a_{T}\leq0$.
Then the optimal prediction for the last round $T$ is 
\begin{align}
\hat{y}_{T}=\mathbf{b}_{T-1}^{\top}\mathbf{A}_{T-1}^{-1}\mathbf{x}_{T}
~. \label{laststep_minmax_optimal}
\end{align}
\end{theorem}
%
%The proof appears in \secref{proof_theorem1}
The proof of the theorem makes use of the following technical lemma.
\begin{lemma}
\label{lem:lemma2}
For all $t=1,2,\ldots,T$
\begin{equation}
a_{t}^{2}\mathbf{x}_{t}^{\top}\mathbf{A}_{t}^{-1}\mathbf{x}_{t}+1-a_{t}=\frac{1+a_{t}\mathbf{x}_{t}^{\top}\mathbf{A}_{t-1}^{-1}\mathbf{x}_{t}-a_{t}}{1+a_{t}\mathbf{x}_{t}^{\top}\mathbf{A}_{t-1}^{-1}\mathbf{x}_{t}}
~. \label{lemma2}
\end{equation}
\end{lemma}
The proof appears in \ref{proof_lemma2}. We now prove
\thmref{thm:theorem1}.
\begin{proof}
The adversary can choose any $y_{T}$, thus the algorithm should predict
$\hat{y}_{T}$ such that the following quantity is minimal,
\begin{eqnarray*}
&&\max_{y_{T}}\left(\sum_{t=1}^{T}\left(y_{t}-\hat{y}_{t}\right)^{2}-\inf_{\mathbf{u}\in\mathbb{R}^{d}}\left(b\left\Vert \mathbf{u}\right\Vert ^{2}+\sum_{t=1}^{T}a_{t}\left(y_{t}-\mathbf{u}^{\top}\mathbf{x}_{t}\right)^{2}\right)\right)\\
&\overset{\eqref{optimal_solution}}{=}&\max_{y_{T}}\left(\sum_{t=1}^{T}\left(y_{t}-\hat{y}_{t}\right)^{2}-\sum_{t=1}^{T}a_{t}y_{t}^{2}+\mathbf{b}_{T}^{\top}\mathbf{A}_{T}^{-1}\mathbf{b}_{T}\right) ~.
\end{eqnarray*}
That is, we need to solve the following minmax problem 
\[
\min_{\hat{y}_{T}}\max_{y_{T}}\left(\sum_{t=1}^{T}\left(y_{t}-\hat{y}_{t}\right)^{2}-\sum_{t=1}^{T}a_{t}y_{t}^{2}+\mathbf{b}_{T}^{\top}\mathbf{A}_{T}^{-1}\mathbf{b}_{T}\right)~.
\]
We use the following relation to re-write the optimization problem, 
\begin{eqnarray}
\mathbf{b}_{T}^{\top}\mathbf{A}_{T}^{-1}\mathbf{b}_{T} 
%& = & \left(\mathbf{b}_{T-1}+a_{T}y_{T}\mathbf{x}_{T}\right)^{\top}\mathbf{A}_{T}^{-1}\left(\mathbf{b}_{T-1}+a_{T}y_{T}\mathbf{x}_{T}\right)\nonumber \\
 & = &
 \mathbf{b}_{T-1}^{\top}\mathbf{A}_{T}^{-1}\mathbf{b}_{T-1}+2a_{T}y_{T}\mathbf{b}_{T-1}^{\top}\mathbf{A}_{T}^{-1}\mathbf{x}_{T}+a_{T}^{2}y_{T}^{2}\mathbf{x}_{T}^{\top}\mathbf{A}_{T}^{-1}\mathbf{x}_{T} ~.\label{t3}
\end{eqnarray}
Omitting all terms that are
not depending on $y_{T}$ and $\hat{y}_{T}$,
\[
\min_{\hat{y}_{T}}\max_{y_{T}}\left(\left(y_{T}-\hat{y}_{T}\right)^{2}-a_{T}y_{T}^{2}+2a_{T}y_{T}\mathbf{b}_{T-1}^{\top}\mathbf{A}_{T}^{-1}\mathbf{x}_{T}+a_{T}^{2}y_{T}^{2}\mathbf{x}_{T}^{\top}\mathbf{A}_{T}^{-1}\mathbf{x}_{T}\right)~.
\]
We manipulate the last problem to be of form \eqref{minmax_objective} using \lemref{lem:lemma2},
% \[
% \min_{\hat{y}_{T}}\max_{y_{T}}\left(\left(a_{T}^{2}\mathbf{x}_{T}^{\top}\mathbf{A}_{T}^{-1}\mathbf{x}_{T}+1-a_{T}\right)y_{T}^{2}+2y_{T}\left(a_{T}\mathbf{b}_{T-1}^{\top}\mathbf{A}_{T}^{-1}\mathbf{x}_{T}-\hat{y}_{T}\right)+\hat{y}_{T}^{2}\right)
% \]
% and by lemma 2 we get
\begin{align}
\min_{\hat{y}_{T}}\max_{y_{T}} \left(\!
  \frac{1+a_{T}\mathbf{x}_{T}^{\top}\mathbf{A}_{T-1}^{-1}\mathbf{x}_{T}-a_{T}}{1+a_{T}\mathbf{x}_{T}^{\top}\mathbf{A}_{T-1}^{-1}\mathbf{x}_{T}}y_{T}^{2}+2y_{T}\left(a_{T}\mathbf{b}_{T-1}^{\top}\mathbf{A}_{T}^{-1}\mathbf{x}_{T}-\hat{y}_{T}\right)+\hat{y}_{T}^{2}
\!\right),\label{minmax}
\end{align}
where 
\[
\alpha(a_T)=\frac{1+a_{T}\mathbf{x}_{T}^{\top}\mathbf{A}_{T-1}^{-1}\mathbf{x}_{T}-a_{T}}{1+a_{T}\mathbf{x}_{T}^{\top}\mathbf{A}_{T-1}^{-1}\mathbf{x}_{T}}
\quad\textrm{ and }\quad
\beta(a_T,\hyi{T})=a_{T}\mathbf{b}_{T-1}^{\top}\mathbf{A}_{T}^{-1}\mathbf{x}_{T}-\hat{y}_{T}~.
\]

We consider two cases:  (1)
$1+a_{T}\mathbf{x}_{T}^{\top}\mathbf{A}_{T-1}^{-1}\mathbf{x}_{T}-a_{T}<0$ 
(corresponding to the middle panel of \figref{fig:minmax}),
and (2)
$1+a_{T}\mathbf{x}_{T}^{\top}\mathbf{A}_{T-1}^{-1}\mathbf{x}_{T}-a_{T}=0$ 
(corresponding to the right panel of \figref{fig:minmax}),
starting with the first case,
\begin{equation}
1+a_{T}\mathbf{x}_{T}^{\top}\mathbf{A}_{T-1}^{-1}\mathbf{x}_{T}-a_{T}<0\label{option1}~.
\end{equation}
Denote the inner-maximization problem by,
\[
f\left(y_{T}\right)\!=\!\frac{1+a_{T}\mathbf{x}_{T}^{\top}\mathbf{A}_{T-1}^{-1}\mathbf{x}_{T}-a_{T}}{1+a_{T}\mathbf{x}_{T}^{\top}\mathbf{A}_{T-1}^{-1}\mathbf{x}_{T}}y_{T}^{2}+2y_{T}\left(a_{T}\mathbf{b}_{T-1}^{\top}\mathbf{A}_{T}^{-1}\mathbf{x}_{T}-\hat{y}_{T}\right)+\hat{y}_{T}^{2}~. 
\]
This function is strictly-concave with respect to $y_{T}$ because of
\eqref{option1}. Thus, it has a unique maximal value given by,
\begin{eqnarray*}
f^{max}(\hyi{T}) 
%& = & \hat{y}_{T}^{2}-\frac{\left(a_{T}\mathbf{b}_{T-1}^{\top}\mathbf{A}_{T}^{-1}\mathbf{x}_{T}-\hat{y}_{T}\right)^{2}\left(1+a_{T}\mathbf{x}_{T}^{\top}\mathbf{A}_{T-1}^{-1}\mathbf{x}_{T}\right)}{1+a_{T}\mathbf{x}_{T}^{\top}\mathbf{A}_{T-1}^{-1}\mathbf{x}_{T}-a_{T}}\\
 & = & -\frac{a_{T}}{1+a_{T}\mathbf{x}_{T}^{\top}\mathbf{A}_{T-1}^{-1}\mathbf{x}_{T}-a_{T}}\hat{y}_{T}^{2}+\frac{2a_{T}\mathbf{b}_{T-1}^{\top}\mathbf{A}_{T}^{-1}\mathbf{x}_{T}\left(1+a_{T}\mathbf{x}_{T}^{\top}\mathbf{A}_{T-1}^{-1}\mathbf{x}_{T}\right)}{1+a_{T}\mathbf{x}_{T}^{\top}\mathbf{A}_{T-1}^{-1}\mathbf{x}_{T}-a_{T}}\hat{y}_{T}\\
 &  & -\frac{\left(a_{T}\mathbf{b}_{T-1}^{\top}\mathbf{A}_{T}^{-1}\mathbf{x}_{T}\right)^{2}\left(1+a_{T}\mathbf{x}_{T}^{\top}\mathbf{A}_{T-1}^{-1}\mathbf{x}_{T}\right)}{1+a_{T}\mathbf{x}_{T}^{\top}\mathbf{A}_{T-1}^{-1}\mathbf{x}_{T}-a_{T}}~.
\end{eqnarray*}
Next, we solve $\min_{\hat{y}_{T}} f^{max}( \hyi{T} )$, which is strictly-convex 
with respect to $\hat{y}_{T}$ because of \eqref{option1}. Solving this problem we get the optimal
last step minmax predictor,
\begin{equation}
\hat{y}_{T}=\mathbf{b}_{T-1}^{\top}\mathbf{A}_{T}^{-1}\mathbf{x}_{T}\left(1+a_{T}\mathbf{x}_{T}^{\top}\mathbf{A}_{T-1}^{-1}\mathbf{x}_{T}\right)
~.
\label{t1}
\end{equation}
We further derive the last equation. From \eqref{Adef} we have,
\begin{equation}
\mathbf{A}_{T}^{-1}a_{T}\mathbf{x}_{T}\mathbf{x}_{T}^{\top}\mathbf{A}_{T-1}^{-1}=\mathbf{A}_{T}^{-1}\left(\mathbf{A}_{T}-\mathbf{A}_{T-1}\right)\mathbf{A}_{T-1}^{-1}=\mathbf{A}_{T-1}^{-1}-\mathbf{A}_{T}^{-1}\label{t2}~.
\end{equation}
Substituting \eqref{t2} in \eqref{t1} we have the following equality
as desired,
\begin{align}
\hat{y}_{T} & =  \mathbf{b}_{T-1}^{\top}\mathbf{A}_{T}^{-1}\mathbf{x}_{T}+\mathbf{b}_{T-1}^{\top}\mathbf{A}_{T}^{-1}a_{T}\mathbf{x}_{T}\mathbf{x}_{T}^{\top}\mathbf{A}_{T-1}^{-1}\mathbf{x}_{T}
 %& =  \mathbf{b}_{T-1}^{\top}\mathbf{A}_{T}^{-1}\mathbf{x}_{T}+\mathbf{b}_{T-1}^{\top}\left(\mathbf{A}_{T-1}^{-1}-\mathbf{A}_{T}^{-1}\right)\mathbf{x}_{T}\\
  =  \mathbf{b}_{T-1}^{\top}\mathbf{A}_{T-1}^{-1}\mathbf{x}_{T}~. \label{t3_thm}
\end{align}
%as desired.

We now move to the second case for which,
\(
1+a_{T}\mathbf{x}_{T}^{\top}\mathbf{A}_{T-1}^{-1}\mathbf{x}_{T}-a_{T}=0~,
\)
which is written equivalently as,
\begin{equation}
a_{T}=\frac{1}{1-\mathbf{x}_{T}^{\top}\mathbf{A}_{T-1}^{-1}\mathbf{x}_{T}}
~. \label{aT}
\end{equation}
Substituting \eqref{aT} in \eqref{minmax} we get, 
\[
\min_{\hat{y}_{T}}\max_{y_{T}}\left(2y_{T}\left(a_{T}\mathbf{b}_{T-1}^{\top}\mathbf{A}_{T}^{-1}\mathbf{x}_{T}-\hat{y}_{T}\right)+\hat{y}_{T}^{2}\right) ~.
\]
For $\hat{y}_{T}\neq
a_{T}\mathbf{b}_{T-1}^{\top}\mathbf{A}_{T}^{-1}\mathbf{x}_{T}$, the
value of the optimization problem is not-bounded as the adversary
may choose $\yi{T} =z^2
\left(a_{T}\mathbf{b}_{T-1}^{\top}\mathbf{A}_{T}^{-1}\mathbf{x}_{T}-\hat{y}_{T}\right)
$ for $z\rightarrow\infty$. Thus, the optimal last step minmax prediction
is to set
$\hat{y}_{T}=a_{T}\mathbf{b}_{T-1}^{\top}\mathbf{A}_{T}^{-1}\mathbf{x}_{T}$.
Substituting $a_T =
1+a_{T}\mathbf{x}_{T}^{\top}\mathbf{A}_{T-1}^{-1}\mathbf{x}_{T}$ and
following the derivation from \eqref{t1} to \eqref{t3_thm} above, yields the
desired identity. 
%
% To show that it is equal to
% $\mathbf{b}_{T-1}^{\top}\mathbf{A}_{T-1}^{-1}\mathbf{x}_{T}$ we use
% the Woodbury identity and get
% \begin{eqnarray*}
% \mathbf{A}_{T}^{-1} & = & \mathbf{A}_{T-1}^{-1}-\frac{\mathbf{A}_{T-1}^{-1}\mathbf{x}_{T}\mathbf{x}_{T}^{\top}\mathbf{A}_{T-1}^{-1}}{\frac{1}{a_{T}}+\mathbf{x}_{T}^{\top}\mathbf{A}_{T-1}^{-1}\mathbf{x}_{T}}\\
%  & \overset{\eqref{aT}}{=} & \mathbf{A}_{T-1}^{-1}-\frac{\mathbf{A}_{T-1}^{-1}\mathbf{x}_{T}\mathbf{x}_{T}^{\top}\mathbf{A}_{T-1}^{-1}}{1-\mathbf{x}_{T}^{\top}\mathbf{A}_{T-1}^{-1}\mathbf{x}_{T}+\mathbf{x}_{T}^{\top}\mathbf{A}_{T-1}^{-1}\mathbf{x}_{T}}\\
%  & = & \mathbf{A}_{T-1}^{-1}-\mathbf{A}_{T-1}^{-1}\mathbf{x}_{T}\mathbf{x}_{T}^{\top}\mathbf{A}_{T-1}^{-1}
% \end{eqnarray*}
%  therefore $\hat{y}_{T}$ can be written as
% \begin{eqnarray*}
% \hat{y}_{T} & = & a_{T}\mathbf{b}_{T-1}^{\top}\left(\mathbf{A}_{T-1}^{-1}-\mathbf{A}_{T-1}^{-1}\mathbf{x}_{T}\mathbf{x}_{T}^{\top}\mathbf{A}_{T-1}^{-1}\right)\mathbf{x}_{T}\\
%  & \overset{\eqref{aT}}{=} & \frac{1}{1-\mathbf{x}_{T}^{\top}\mathbf{A}_{T-1}^{-1}\mathbf{x}_{T}}\mathbf{b}_{T-1}^{\top}\mathbf{A}_{T-1}^{-1}\mathbf{x}_{T}\left(1-\mathbf{x}_{T}^{\top}\mathbf{A}_{T-1}^{-1}\mathbf{x}_{T}\right)\\
%  & = & \mathbf{b}_{T-1}^{\top}\mathbf{A}_{T-1}^{-1}\mathbf{x}_{T}
% \end{eqnarray*}
 \QED\end{proof}
We conclude by noting that although we did not restrict the form of
the predictor $\hyi{T}$, it turns out that it is a linear predictor
defined by $\hyi{T} = \vxti{T}\vwi{T-1}$ for $\vwi{T-1} =
\mathbf{A}_{T-1}^{-1}\mathbf{b}_{T-1}$. In other words, the functional
form of the optimal predictor is the same as the form of the
comparison function class - linear functions in our case. We call the
algorithm (defined using \eqref{Adef}, \eqref{bdef} and
\eqref{laststep_minmax_optimal}) \texttt{WEMM} for weighted min-max
prediction. 
We note that \texttt{WEMM}
can also be seen as an incremental off-line
algorithm~\cite{AzouryWa01} or follow-the-leader, on a
weighted sequence. The prediction $\hyi{T} =
\vxti{T}\vwi{T-1}$ is with a model that is optimal over a prefix of length $T-1$.  The prediction of the optimal
predictor defined in \eqref{optimal_solution} is
$\vxti{T}\mathbf{u}_{T-1}=\vxti{T}\mathbf{A}_{T-1}^{-1}\mathbf{b}_{T-1}=\hat{y}_T$,
where $\hat{y}_T$ was defined in \eqref{laststep_minmax_optimal}.

\subsection{Recursive form}
\label{Recursive_form}
Although \thmref{thm:theorem1} is correct for
$1+a_{T}\mathbf{x}_{T}^{\top}\mathbf{A}_{T-1}^{-1}\mathbf{x}_{T}-a_{T}\leq0$,
in the rest of the paper we will (almost always) assume an equality, that is
\begin{align}
a_{t}=\frac{1}{1-\mathbf{x}_{t}^{\top}\mathbf{A}_{t-1}^{-1}\mathbf{x}_{t}} \quad,\quad t=1 \dots T~.
\label{aa_t}
\end{align}
For this case, \texttt{WEMM} algorithm can be expressed in a recursive form in terms of weight vector $\mathbf{w}_{t}$ and a covariance-like matrix $\mathbf{\Sigma}_{t}$. We denote $\mathbf{w}_{t}=\mathbf{A}_{t}^{-1}\mathbf{b}_{t}$ and $\mathbf{\Sigma}_{t}=\mathbf{A}_{t}^{-1}$, and develop recursive update rules for $\mathbf{w}_{t}$ and $\mathbf{\Sigma}_{t}$:
\begin{eqnarray}
\mathbf{w}_{t} & = & \mathbf{A}_{t}^{-1}\mathbf{b}_{t} \nonumber \\
 & = & \left(\mathbf{A}_{t-1}+a_{t}\mathbf{x}_{t}\mathbf{x}_{t}^{\top}\right)^{-1}\left(\mathbf{b}_{t-1}+a_{t}y_{t}\mathbf{x}_{t}\right) \nonumber \\
 & = & \left(\mathbf{A}_{t-1}^{-1}-\frac{\mathbf{A}_{t-1}^{-1}\mathbf{x}_{t}\mathbf{x}_{t}^{\top}\mathbf{A}_{t-1}^{-1}}{a_{t}^{-1}+\mathbf{x}_{t}^{\top}\mathbf{A}_{t-1}^{-1}\mathbf{x}_{t}}\right)\left(\mathbf{b}_{t-1}+a_{t}y_{t}\mathbf{x}_{t}\right) \nonumber \\
 & = & \mathbf{w}_{t-1}-\frac{\mathbf{A}_{t-1}^{-1}\mathbf{x}_{t}\mathbf{x}_{t}^{\top}\mathbf{w}_{t-1}}{a_{t}^{-1}+\mathbf{x}_{t}^{\top}\mathbf{A}_{t-1}^{-1}\mathbf{x}_{t}}+a_{t}y_{t}\mathbf{A}_{t-1}^{-1}\mathbf{x}_{t}-\frac{a_{t}y_{t}\mathbf{A}_{t-1}^{-1}\mathbf{x}_{t}\mathbf{x}_{t}^{\top}\mathbf{A}_{t-1}^{-1}\mathbf{x}_{t}}{a_{t}^{-1}+\mathbf{x}_{t}^{\top}\mathbf{A}_{t-1}^{-1}\mathbf{x}_{t}} \nonumber \\
 & = & \mathbf{w}_{t-1}+\frac{y_{t}\mathbf{A}_{t-1}^{-1}\mathbf{x}_{t}-\mathbf{A}_{t-1}^{-1}\mathbf{x}_{t}\mathbf{x}_{t}^{\top}\mathbf{w}_{t-1}}{a_{t}^{-1}+\mathbf{x}_{t}^{\top}\mathbf{A}_{t-1}^{-1}\mathbf{x}_{t}} \nonumber \\
 & \overset{\eqref{aa_t}}{=} & \mathbf{w}_{t-1}+\left(y_{t}-\mathbf{x}_{t}^{\top}\mathbf{w}_{t-1}\right)\mathbf{A}_{t-1}^{-1}\mathbf{x}_{t} \nonumber \\
 & = & \mathbf{w}_{t-1}+\left(y_{t}-\mathbf{x}_{t}^{\top}\mathbf{w}_{t-1}\right)\mathbf{\Sigma}_{t-1}\mathbf{x}_{t}~, \label{w_t}
\end{eqnarray} 
and
\begin{eqnarray*}
\mathbf{\Sigma}_{t}^{-1} & = & \mathbf{A}_{t}=\mathbf{A}_{t-1}+a_{t}\mathbf{x}_{t}\mathbf{x}_{t}^{\top}  \\
 & \overset{\eqref{aa_t}}{=} & \mathbf{A}_{t-1}+\frac{\mathbf{x}_{t}\mathbf{x}_{t}^{\top}}{1-\mathbf{x}_{t}^{\top}\mathbf{A}_{t-1}^{-1}\mathbf{x}_{t}} \\
 & = & \mathbf{\Sigma}_{t-1}^{-1}+\frac{\mathbf{x}_{t}\mathbf{x}_{t}^{\top}}{1-\mathbf{x}_{t}^{\top}\mathbf{\Sigma}_{t-1}\mathbf{x}_{t}} \end{eqnarray*}
or
\begin{eqnarray}
\mathbf{\Sigma}_{t} & = & \mathbf{\Sigma}_{t-1}-\mathbf{\Sigma}_{t-1}\mathbf{x}_{t}\mathbf{x}_{t}^{\top}\mathbf{\Sigma}_{t-1}~. \label{sigma_t}
\end{eqnarray}
A summary of the algorithm in a recursive form appears in the right column of \tabref{table:algorithms}.

%\kc{add a relation to AROW, and other algorithms of similar form}
%\edward{Done. Please fix the columns of \tabref{table:algorithms}. Thanks.}

It is instructive to compare similar second order online algorithms for regression. The ridge-regression~\cite{Foster91}, summarized in the third column of \tabref{table:algorithms}, uses the previous examples to generate a weight-vector, which is used to predict current example. On round $t$ it sets a weight-vector to be the solution of the following optimization problem,
\[
\mathbf{w}_{t-1}=\underset{\mathbf{w}}{\arg\min}\left[\sum_{i=1}^{t-1}\left(y_{i}-\mathbf{x}_{i}^{\top}\mathbf{w}\right)^{2}+b\left\Vert \mathbf{w}\right\Vert ^{2}\right]~,
\]
and outputs a prediction $\hat{y}_{t}=\mathbf{x}_{t}^{\top}\mathbf{w}_{t-1}$.
The recursive least squares (RLS)~\cite{Hayes} is a similar algorithm, yet it uses a forgetting factor $0<r \leq 1$, and sets the weight-vector according to
\[
\mathbf{w}_{t-1}=\underset{\mathbf{w}}{\arg\min}\left[\sum_{i=1}^{t-1}r^{t-i-1}\left(y_{i}-\mathbf{x}_{i}^{\top}\mathbf{w}\right)^{2}\right]~.
\]
The Aggregating Algorithm for regression (AAR)~\cite{Vovk01},
summarized in the second column of \tabref{table:algorithms}, was
introduced by Vovk and it is similar to ridge-regression, except it
contains additional regularization, which eventually makes it shrink
the predictions. It is an application of the Aggregating Algorithm~\cite{vovkAS} (a general algorithm for merging prediction strategies) to the problem of linear regression with square loss. On round $t$, the weight-vector is obtained according to
\[
\mathbf{w}_{t}=\underset{\mathbf{w}}{\arg\min}\left[\sum_{i=1}^{t-1}\left(y_{i}-\mathbf{x}_{i}^{\top}\mathbf{w}\right)^{2}+\left(\mathbf{x}_{t}^{\top}\mathbf{w}\right)^{2}+b\left\Vert \mathbf{w}\right\Vert ^{2}\right]~,
\]
and the algorithm predicts $\hat{y}_{t}=\mathbf{x}_{t}^{\top}\mathbf{w}_{t}$. Compared to ridge-regression, the AAR algorithm uses an additional input pair $(\mathbf{x}_t,0)$. The AAR algorithm was shown to be last-step min-max optimal by Forster~\cite{Forster}, that is the predictions can be obtained by solving \eqref{minmax_algorithm_1} for $a_t=1,~ t=1\comdots T$.

The AROWR algorithm~\cite{VaitsCr11,CrammerKuDr12}, summarized in the left column of \tabref{table:algorithms}, is a modification of the
AROW algorithm~\cite{CrammerKuDr09} for regression. It maintains a Gaussian distribution parameterized by a
mean $\mathbf{w}_t\in\reals^d$ and a full covariance matrix
$\mathbf{\Sigma}_t\in\reals^{d \times d}$. Intuitively, the mean $\mathbf{w}_t$
represents a current linear function, while the covariance matrix
$\mathbf{\Sigma}_t$ captures the uncertainty in the linear function
$\mathbf{w}_t$. Given a new example $(\mathbf{x}_t,y_t)$ the algorithm uses its
current mean to make a prediction $\hat{y}_{t}=\mathbf{x}_{t}^{\top}\mathbf{w}_{t-1}$.
AROWR then sets the new distribution to be the solution of the
following optimization problem,
\begin{align*}
%\mcal{O}\paren{\vw, \msigma}=
 \arg\min_{\mathbf{w}, \mathbf{\Sigma}} \left[ \KL\paren{ \norm\paren{\mathbf{w}, \mathbf{\Sigma}} \,\Vert\,
    \norm\paren{\mathbf{w}_{t-1}, \mathbf{\Sigma}_{t-1}}}  + \frac{1}{2r}
  \paren{\yi{t} - \mathbf{w}^{\top}\vxii}^2+ \frac{1}{2r} \paren{\vxti{t}
    \mathbf{\Sigma} \vxi{t}} \right]~.
\end{align*}
Crammer et.al.~\cite{CrammerKuDr12} derived regret bounds for this algorithm.

Comparing \texttt{WEMM} to other algorithms we note two
differences. First, for the weight-vector update rule, we do not have the normalization term $1+\mathbf{x}_{t}^{\top}\mathbf{\Sigma}_{t-1}\mathbf{x}_{t}$. Second, for the covariance matrix update rule, our algorithm gives non-constant scale to the increment by $\mathbf{x}_{t}\mathbf{x}_{t}^{\top}$. This scale $1/(1-\mathbf{x}_{t}^{\top}\mathbf{\Sigma}_{t-1}\mathbf{x}_{t})$ is small when the current instance $\mathbf{x}_{t}$ lies along the directions spanned by previously observed inputs $\left\{ \mathbf{x}_{i}\right\} _{i=1}^{t-1}$, and large when the current instance $\mathbf{x}_{t}$ lies along previously unobserved directions.
 
%\begin{figure}[t!]
%{
%\paragraph{Parameter:} $1<b$
%\paragraph{Initialize:} Set
%$\mathbf{w}_0=\vzero\in\reals^d$ and %$\mathbf{\Sigma}_{0}=b^{-1}\,\mathbf{I}\in\reals^{d\times d}$\\
%{\bf For $t=1 \comdots T$} do
%\begin{itemize}
%\nolineskips
%\item Receive an instance $\vxi{t}$
%\item Output  prediction $\hyi{t}=\mathbf{x}_{t}^{\top}\mathbf{w}_{t-1}$
%\item Receive the correct label $\yi{t}$
%\item
%Update:
%\begin{align}
%\mathbf{w}_{t}&=\mathbf{w}_{t-1}+\left(y_{t}-\mathbf{x}_{t}^{\top}\mathbf{w}_{t-1}\right)\mathbf{\Sigma}_{t-1}\mathbf{x}_{t}
%\label{w_t}\\
%\mathbf{\Sigma}_{t}&=\mathbf{\Sigma}_{t-1}-\mathbf{\Sigma}_{t-1}\mathbf{x}_{t}\mathbf{x}_{t}^{\top}\mathbf{\Sigma}_{t-1}
%\label{sigma_t}
%\end{align}
%\end{itemize}
%\paragraph{Output:}  $\mathbf{w}_{T} \ ,\ \mathbf{\Sigma}_{T}$\\
%}
%\figline
%\caption{WEMM: weighted min-max prediction.}
%\label{algorithm:WEMM}
%\end{figure}

\begin{center}
\begin{table*}[ht]
{\small
\hfill{}
\begin{tabulary}{1.15\textwidth}{|C|C|L|L|L|L|} % centered columns (6 columns)
\hline                       %inserts double horizontal lines
 &  & \textbf{AROWR}~\cite{VaitsCr11,CrammerKuDr12} &
 \textbf{AAR}~\cite{Vovk01} / \textbf{Min-Max}~\cite{Forster}  &
 \textbf{Ridge-Regression}~\cite{Foster91} & \textbf{WEMM} this work\\ [0.5ex] % inserts table
%heading
\hline                  % inserts single horizontal line
 Parameters  & & $0<r,b$ & $0<b$ & $0<b$  &  $1<b$ \\ % inserting body of the table
\hline
Initialize & & \multicolumn{4}{c|}{ $\mathbf{w}_{0}=\mathbf{0}$~,~$\mathbf{\Sigma}_{0}=b^{-1}\mathbf{I}$ } \\
\cline{2-6}
% Initialize &  &
% $w_{0}=0$~,~$\mathbf{\Sigma}_{0}=b^{-1}\mathbf{I}$ &
% $w_{0}=0$~,~$\mathbf{\Sigma}_{0}=b^{-1}\mathbf{I}$ & %$w_{0}=0$~,~$\mathbf{\Sigma}_{0}=b^{-1}\mathbf{I}$ & %$w_{0}=0$~,~$\mathbf{\Sigma}_{0}=b^{-1}\mathbf{I}$ \\ [0.5ex]
\hline
%\cline{2-5}
 & & \multicolumn{4}{c|}{ Receive an instance $\mathbf{x}_{t}$}  \\
\cline{2-6}
For $t=1 ... T$ & Output prediction & \[\hat{y}_{t}=\mathbf{x}_{t}^{\top}\mathbf{w}_{t-1}\] & \[\!\!\!\hat{y}_{t}\!\!=\!\!\frac{\mathbf{x}_{t}^{\top}\mathbf{w}_{t-1}}{1+\mathbf{x}_{t}^{\top}\mathbf{\Sigma}_{t-1}\mathbf{x}_{t}}\] & \[\hat{y}_{t}=\mathbf{x}_{t}^{\top}\mathbf{w}_{t-1}\]
& \[\hat{y}_{t}=\mathbf{x}_{t}^{\top}\mathbf{w}_{t-1}\] \\
%$t=1 ... T$ & prediction & & &\\
\cline{2-6}
 & & \multicolumn{4}{c|}{Receive a correct label $y_{t}$ }  \\
 \cline{2-6}
 & Update $\mathbf{\Sigma}_{t}$: & \[\begin{array}{ll}\mathbf{\Sigma}_{t}^{-1}=\\\mathbf{\Sigma}_{t-1}^{-1}+\frac{1}{r}\mathbf{x}_{t}\mathbf{x}_{t}^{\top}\end{array}\] & \[\begin{array}{ll}\mathbf{\Sigma}_{t}^{-1}=\\\mathbf{\Sigma}_{t-1}^{-1}+\mathbf{x}_{t}\mathbf{x}_{t}^{\top}\end{array}\] & \[\begin{array}{ll}\mathbf{\Sigma}_{t}^{-1}=\\\mathbf{\Sigma}_{t-1}^{-1}+\mathbf{x}_{t}\mathbf{x}_{t}^{\top}\end{array}\] 
&
\[\!\!\!\begin{array}{ll}\mathbf{\Sigma}_{t}^{-1}=\\\mathbf{\Sigma}_{t-1}^{-1}+\frac{\mathbf{x}_{t}\mathbf{x}_{t}^{\top}}{1-\mathbf{x}_{t}^{\top}\mathbf{\Sigma}_{t-1}\mathbf{x}_{t}}\end{array}\]  \\
\cline{2-6}
 & Update $\mathbf{w}_{t}$: & \vspace{0.2cm}\[\!\!\!\!\!\begin{array}{ll}\mathbf{w}_{t} \!=\!  \mathbf{w}_{t-1}\\+\frac{\left(y_{t}-\mathbf{x}_{t}^{\top}\mathbf{w}_{t-1}\right)\mathbf{\Sigma}_{t-1}\mathbf{x}_{t}}{r+\mathbf{x}_{t}^{\top}\mathbf{\Sigma}_{t-1}\mathbf{x}_{t}}\end{array}\]
 & \[\!\!\!\!\!\!\begin{array}{ll}\mathbf{w}_{t}  \!= \mathbf{w}_{t-1}\\+\frac{\left(y_{t}-\mathbf{x}_{t}^{\top}\mathbf{w}_{t-1}\right)\mathbf{\Sigma}_{t-1}\mathbf{x}_{t}}{1+\mathbf{x}_{t}^{\top}\mathbf{\Sigma}_{t-1}\mathbf{x}_{t}}\end{array}\]
& \[\!\!\!\!\!\!\begin{array}{ll}\mathbf{w}_{t}  \!= \mathbf{w}_{t-1}\\+\frac{\left(y_{t}-\mathbf{x}_{t}^{\top}\mathbf{w}_{t-1}\right)\mathbf{\Sigma}_{t-1}\mathbf{x}_{t}}{1+\mathbf{x}_{t}^{\top}\mathbf{\Sigma}_{t-1}\mathbf{x}_{t}}\end{array}\] & \[\!\!\!\!\!\!\begin{array}{ll}\mathbf{w}_{t}  \!= \mathbf{w}_{t-1}\\+\left(y_{t}-\mathbf{x}_{t}^{\top}\mathbf{w}_{t-1}\right)\mathbf{\Sigma}_{t-1}\mathbf{x}_{t}\end{array}\] \\
 \hline
Output & & \multicolumn{4}{c|}{ $\mathbf{w}_{T} \ ,\ \mathbf{\Sigma}_{T}$ } \\
% [1ex]     % [1ex] adds vertical space
\hline  %inserts single line
\end{tabulary}}
\caption{Second order online algorithms for regression} % title of Table
\hfill{}
\label{table:algorithms} 
\end{table*}
\end{center}

\subsection{Kernel version of the algorithm}
\label{kernel_form}
In this section we show that the \texttt{WEMM} algorithm can be expressed in dual variables, which allows an efficient run of the algorithm in any reproducing kernel Hilbert space.
We show by induction that the weight-vector $\mathbf{w}_{t}$ and
the covariance matrix $\mathbf{\Sigma}_{t}$ computed by the \texttt{WEMM}
algorithm in the right column of \tabref{table:algorithms} can be written in the form
\begin{eqnarray*}
\mathbf{w}_{t} & = & \sum_{i=1}^{t}\alpha_{i}^{\left(t\right)}\mathbf{x}_{i}\\
\mathbf{\Sigma}_{t} & = & \sum_{j=1}^{t}\sum_{k=1}^{t}\beta_{j,k}^{\left(t\right)}\mathbf{x}_{j}\mathbf{x}_{k}^{\top}+b^{-1}\mathbf{I}~,
\end{eqnarray*}
 where the coefficients $\alpha_{i}$ and $\beta_{j,k}$ depend only
on inner products of the input vectors.

For the initial step we have $\mathbf{w}_{0}=\mathbf{0}$ and $\mathbf{\Sigma}_{0}=b^{-1}\mathbf{I}$
which are trivially written in the desired form by setting $\alpha^{\left(0\right)}=0$
and $\beta^{\left(0\right)}=0$. We proceed to the induction step.
From the weight-vector update rule \eqref{w_t} we get
\begin{eqnarray*}
\mathbf{w}_{t} & = & \mathbf{w}_{t-1}+\left(y_{t}-\mathbf{x}_{t}^{\top}\mathbf{w}_{t-1}\right)\mathbf{\Sigma}_{t-1}\mathbf{x}_{t}=\\
 & = & \sum_{i=1}^{t-1}\alpha_{i}^{\left(t-1\right)}\mathbf{x}_{i}+\left(y_{t}-\mathbf{x}_{t}^{\top}\sum_{i=1}^{t-1}\alpha_{i}^{\left(t-1\right)}\mathbf{x}_{i}\right)\left(\sum_{j=1}^{t-1}\sum_{k=1}^{t-1}\beta_{j,k}^{\left(t-1\right)}\mathbf{x}_{j}\mathbf{x}_{k}^{\top}+b^{-1}\mathbf{I}\right)\mathbf{x}_{t}\\
 & = & \sum_{i=1}^{t-1}\alpha_{i}^{\left(t-1\right)}\mathbf{x}_{i}+\left(y_{t}-\sum_{i=1}^{t-1}\alpha_{i}^{\left(t-1\right)}\left(\mathbf{x}_{t}^{\top}\mathbf{x}_{i}\right)\right)\left(\sum_{j=1}^{t-1}\sum_{k=1}^{t-1}\beta_{j,k}^{\left(t-1\right)}\left(\mathbf{x}_{k}^{\top}\mathbf{x}_{t}\right)\mathbf{x}_{j}+b^{-1}\mathbf{x}_{t}\right)\\
 & = & \sum_{i=1}^{t-1}\left[\alpha_{i}^{\left(t-1\right)}+\left(y_{t}-\sum_{l=1}^{t-1}\alpha_{l}^{\left(t-1\right)}\left(\mathbf{x}_{t}^{\top}\mathbf{x}_{l}\right)\right)\sum_{k=1}^{t-1}\beta_{i,k}^{\left(t-1\right)}\left(\mathbf{x}_{k}^{\top}\mathbf{x}_{t}\right)\right]\mathbf{x}_{i}+b^{-1}\left(y_{t}-\sum_{i=1}^{t-1}\alpha_{i}^{\left(t-1\right)}\left(\mathbf{x}_{t}^{\top}\mathbf{x}_{i}\right)\right)\mathbf{x}_{t}~,
\end{eqnarray*}
 thus
\[
\alpha_{i}^{\left(t\right)}=\begin{cases}
\alpha_{i}^{\left(t-1\right)}+\left(y_{t}-\sum_{l=1}^{t-1}\alpha_{l}^{\left(t-1\right)}\left(\mathbf{x}_{t}^{\top}\mathbf{x}_{l}\right)\right)\sum_{k=1}^{t-1}\beta_{i,k}^{\left(t-1\right)}\left(\mathbf{x}_{k}^{\top}\mathbf{x}_{t}\right) & ~~ i=1 \comdots t-1\\
b^{-1}\left(y_{t}-\sum_{l=1}^{t-1}\alpha_{l}^{\left(t-1\right)}\left(\mathbf{x}_{t}^{\top}\mathbf{x}_{l}\right)\right) & ~~ i=t
\end{cases}
\]
 From the covariance matrix update rule \eqref{sigma_t} we get
\begin{eqnarray*}
\mathbf{\Sigma}_{t} & = & \mathbf{\Sigma}_{t-1}-\mathbf{\Sigma}_{t-1}\mathbf{x}_{t}\mathbf{x}_{t}^{\top}\mathbf{\Sigma}_{t-1}\\
 & = & \sum_{j=1}^{t-1}\sum_{k=1}^{t-1}\beta_{j,k}^{\left(t-1\right)}\mathbf{x}_{j}\mathbf{x}_{k}^{\top}+b^{-1}\mathbf{I}-\left(\sum_{j=1}^{t-1}\sum_{k=1}^{t-1}\beta_{j,k}^{\left(t-1\right)}\mathbf{x}_{j}\mathbf{x}_{k}^{\top}+b^{-1}\mathbf{I}\right)\mathbf{x}_{t}\mathbf{x}_{t}^{\top}\left(\sum_{j=1}^{t-1}\sum_{k=1}^{t-1}\beta_{j,k}^{\left(t-1\right)}\mathbf{x}_{j}\mathbf{x}_{k}^{\top}+b^{-1}\mathbf{I}\right)\\
 & = & \sum_{j=1}^{t-1}\sum_{k=1}^{t-1}\beta_{j,k}^{\left(t-1\right)}\mathbf{x}_{j}\mathbf{x}_{k}^{\top}+b^{-1}\mathbf{I}-\left(\sum_{j=1}^{t-1}\sum_{k=1}^{t-1}\beta_{j,k}^{\left(t-1\right)}\left(\mathbf{x}_{k}^{\top}\mathbf{x}_{t}\right)\mathbf{x}_{j}+b^{-1}\mathbf{x}_{t}\right)\left(\sum_{j=1}^{t-1}\sum_{k=1}^{t-1}\beta_{j,k}^{\left(t-1\right)}\left(\mathbf{x}_{t}^{\top}\mathbf{x}_{j}\right)\mathbf{x}_{k}^{\top}+b^{-1}\mathbf{x}_{t}^{\top}\right)\\
 & = & \sum_{j=1}^{t-1}\sum_{k=1}^{t-1}\beta_{j,k}^{\left(t-1\right)}\mathbf{x}_{j}\mathbf{x}_{k}^{\top}+b^{-1}\mathbf{I}-\sum_{j=1}^{t-1}\sum_{k=1}^{t-1}\sum_{l=1}^{t-1}\sum_{m=1}^{t-1}\beta_{l,m}^{\left(t-1\right)}\beta_{j,k}^{\left(t-1\right)}\left(\mathbf{x}_{k}^{\top}\mathbf{x}_{t}\right)\left(\mathbf{x}_{t}^{\top}\mathbf{x}_{l}\right)\mathbf{x}_{j}\mathbf{x}_{m}^{\top}\\
 &  & -b^{-1}\sum_{j=1}^{t-1}\sum_{k=1}^{t-1}\beta_{j,k}^{\left(t-1\right)}\left(\mathbf{x}_{t}^{\top}\mathbf{x}_{j}\right)\mathbf{x}_{t}\mathbf{x}_{k}^{\top}-b^{-1}\sum_{j=1}^{t-1}\sum_{k=1}^{t-1}\beta_{j,k}^{\left(t-1\right)}\left(\mathbf{x}_{k}^{\top}\mathbf{x}_{t}\right)\mathbf{x}_{j}\mathbf{x}_{t}^{\top}-b^{-2}\mathbf{x}_{t}\mathbf{x}_{t}^{\top}\\
 & = & \sum_{j=1}^{t-1}\sum_{k=1}^{t-1}\left[\beta_{j,k}^{\left(t-1\right)}-\sum_{l=1}^{t-1}\sum_{m=1}^{t-1}\beta_{l,k}^{\left(t-1\right)}\beta_{j,m}^{\left(t-1\right)}\left(\mathbf{x}_{m}^{\top}\mathbf{x}_{t}\right)\left(\mathbf{x}_{t}^{\top}\mathbf{x}_{l}\right)\right]\mathbf{x}_{j}\mathbf{x}_{k}^{\top}+b^{-1}\mathbf{I}\\
 &  & -b^{-1}\sum_{k=1}^{t-1}\sum_{j=1}^{t-1}\beta_{j,k}^{\left(t-1\right)}\left(\mathbf{x}_{t}^{\top}\mathbf{x}_{j}\right)\mathbf{x}_{t}\mathbf{x}_{k}^{\top}-b^{-1}\sum_{j=1}^{t-1}\sum_{k=1}^{t-1}\beta_{j,k}^{\left(t-1\right)}\left(\mathbf{x}_{k}^{\top}\mathbf{x}_{t}\right)\mathbf{x}_{j}\mathbf{x}_{t}^{\top}-b^{-2}\mathbf{x}_{t}\mathbf{x}_{t}^{\top}~,
\end{eqnarray*}
 thus
\[
\beta_{j,k}^{\left(t\right)}=\begin{cases}
\beta_{j,k}^{\left(t-1\right)}-\sum_{l=1}^{t-1}\sum_{m=1}^{t-1}\beta_{l,k}^{\left(t-1\right)}\beta_{j,m}^{\left(t-1\right)}\left(\mathbf{x}_{m}^{\top}\mathbf{x}_{t}\right)\left(\mathbf{x}_{t}^{\top}\mathbf{x}_{l}\right) & ~~ j,k=1 \comdots t-1\\
-b^{-1}\sum_{l=1}^{t-1}\beta_{l,k}^{\left(t-1\right)}\left(\mathbf{x}_{t}^{\top}\mathbf{x}_{l}\right) & ~~ j=t~,~k=1 \comdots t-1\\
-b^{-1}\sum_{l=1}^{t-1}\beta_{j,l}^{\left(t-1\right)}\left(\mathbf{x}_{l}^{\top}\mathbf{x}_{t}\right) & ~~  k=t~,~j=1 \comdots t-1\\
-b^{-2} & ~~ j=k=t
\end{cases}
\]
 A summary of the kernel version of the \texttt{WEMM} algorithm appears in \figref{algorithm:kernel_WEMM}.

\begin{figure}[t!]
{
\paragraph{Parameter:} $1<b$, kernel function $K:\mathbb{R}^{d}\times\mathbb{R}^{d}\rightarrow\mathbb{R}$
\paragraph{Initialize:} Set
$\alpha^{\left(0\right)}=0$ and $\beta^{\left(0\right)}=0$\\
{\bf For $t=1 \comdots T$} do
\begin{itemize}
\nolineskips
\item Receive an instance $\vxi{t}$
\item Output  prediction $\hyi{t}=\sum_{i=1}^{t-1}\alpha_{i}^{\left(t-1\right)}K\left(\mathbf{x}_{t},\mathbf{x}_{i}\right)$
\item Receive the correct label $\yi{t}$
\item
Update:
\begin{align}
\alpha_{i}^{\left(t\right)}&=\begin{cases}
\alpha_{i}^{\left(t-1\right)}+\left(y_{t}-\sum_{l=1}^{t-1}\alpha_{l}^{\left(t-1\right)}K\left(\mathbf{x}_{t},\mathbf{x}_{l}\right)\right)\sum_{k=1}^{t-1}\beta_{i,k}^{\left(t-1\right)}K\left(\mathbf{x}_{k},\mathbf{x}_{t}\right)
& ~~ i=1 \comdots t-1\\
b^{-1}\left(y_{t}-\sum_{l=1}^{t-1}\alpha_{l}^{\left(t-1\right)}K\left(\mathbf{x}_{t},\mathbf{x}_{l}\right)\right) & ~~ i=t
\end{cases}\\
\beta_{j,k}^{\left(t\right)}&=\begin{cases}
\beta_{j,k}^{\left(t-1\right)}-\sum_{l=1}^{t-1}\sum_{m=1}^{t-1}\beta_{l,k}^{\left(t-1\right)}\beta_{j,m}^{\left(t-1\right)}K\left(\mathbf{x}_{m},\mathbf{x}_{t}\right)K\left(\mathbf{x}_{t},\mathbf{x}_{l}\right)
& ~~ j,k=1 \comdots t-1\\
-b^{-1}\sum_{l=1}^{t-1}\beta_{l,k}^{\left(t-1\right)}K\left(\mathbf{x}_{t},\mathbf{x}_{l}\right)
& ~~ j=t~,~k=1 \comdots t-1\\
-b^{-1}\sum_{l=1}^{t-1}\beta_{j,l}^{\left(t-1\right)}K\left(\mathbf{x}_{l},\mathbf{x}_{t}\right)
& ~~ k=t~,~j=1 \comdots t-1\\
-b^{-2} & ~~ j=k=t
\end{cases}
\end{align}
\end{itemize}
\paragraph{Output:}  $\braces{\alpha_{i}^{\left(T\right)}}_{i=1}^T ,\braces{\beta_{j,k}^{\left(T\right)}}_{j,k=1}^T$\\
}
\figline
\caption{Kernel WEMM}
\label{algorithm:kernel_WEMM}
\end{figure}

\section{Analysis}
We analyze the algorithm in two steps. First, in
\thmref{thm:theorem2} we show that the algorithm suffers a {\em
  constant} regret 
compared with the optimal
weight vector $\vu$ evaluated using {\em the weighted} loss,
$L^{\boldsymbol{a}}(\vu)$. Second, in \thmref{thm:theorem3} and \thmref{thm:theorem4} we
show that the difference of the weighted-loss $L^{\boldsymbol{a}}(\vu)$ to the true loss
$L(\vu)$ is only logarithmic in $T$ or in $L_T(\vu)$.
\begin{theorem}
\label{thm:theorem2}
Assume $1+a_{t}\mathbf{x}_{t}^{\top}\mathbf{A}_{t-1}^{-1}\mathbf{x}_{t}-a_{t}\leq0$
for $t = 1 \dots T$ (which is satisfied by our choice later). Then,
the loss of \texttt{WEMM}, 
$\hat{y}_{t}=\mathbf{b}_{t-1}^{\top}\mathbf{A}_{t-1}^{-1}\mathbf{x}_{t}$ for $t=1 \dots T$, is upper bounded by,
\begin{align*}
L_{T}(\texttt{WEMM})\leq\inf_{\mathbf{u}\in\mathbb{R}^{d}}\left(b\left\Vert
    \mathbf{u}\right\Vert
  ^{2}+L_{T}^{\boldsymbol{a}}(\mathbf{u})\right) ~.
\end{align*}
Furthermore, if
$1+a_{t}\mathbf{x}_{t}^{\top}\mathbf{A}_{t-1}^{-1}\mathbf{x}_{t}-a_{t}
= 0$, then the last inequality is in fact an equality.
\end{theorem}
\begin{proofsketch}
Long algebraic manipulation given in \ref{proof_theorem2}  yields,
\begin{align*}
\ell_{t}(\texttt{WEMM})+\inf_{\mathbf{u}\in\mathbb{R}^{d}}\left(b\left\Vert \mathbf{u}\right\Vert ^{2}+L_{t-1}^{\boldsymbol{a}}(\mathbf{u})\right)-\inf_{\mathbf{u}\in\mathbb{R}^{d}}\left(b\left\Vert \mathbf{u}\right\Vert ^{2}+L_{t}^{\boldsymbol{a}}(\mathbf{u})\right)
=\frac{1+a_{t}\mathbf{x}_{t}^{\top}\mathbf{A}_{t-1}^{-1}\mathbf{x}_{t}-a_{t}}{1+a_{t}\mathbf{x}_{t}^{\top}\mathbf{A}_{t-1}^{-1}\mathbf{x}_{t}}\left(y_{t}-\hat{y}_{t}\right)^{2}\leq0 ~.
\end{align*}
Summing over $t$ gives the desired bound.
\QED
\end{proofsketch}

Next we decompose the weighted loss
$L_{T}^{\boldsymbol{a}}(\mathbf{u})$ into a sum of the actual loss
$L_T(\vu)$ and a logarithmic term. We give two bounds - one is logarithmic in $T$ (\thmref{thm:theorem3}), and the second is logarithmic in $L_T(\vu)$ (\thmref{thm:theorem4}).
We use the following notation of the loss suffered by $\vu$ over the
worst example,
\begin{align}
S=S(\vu)= \sup_{1\leq t\leq T}\ell_{t}(\mathbf{u}),\label{sup_loss}
\end{align}
where clearly $S$ depends explicitly in $\vu$, which is omitted for
simplicity. We now turn to state our first result.
\begin{theorem}
\label{thm:theorem3}
Assume $\left\Vert \mathbf{x}_{t}\right\Vert \leq1$ for $t=1 \dots T$
and $b>1$. Assume further that $a_{t}=\frac{1}{1-\mathbf{x}_{t}^{\top}\mathbf{A}_{t-1}^{-1}\mathbf{x}_{t}}$
for $t=1 \dots T$. Then
\begin{align*}
L_{T}^{\boldsymbol{a}}(\mathbf{u})\leq L_{T}(\mathbf{u})+
\frac{b}{b-1}S\ln\left|\frac{1}{b}\mathbf{A}_{T}\right| ~.
\end{align*}
\end{theorem}
The proof follows similar steps to Forster~\cite{Forster}.
A detailed proof is given in \ref{proof_theorem3}.
\begin{proofsketch}
We decompose the weighted loss, 
\begin{equation}
L_{T}^{\boldsymbol{a}}(\mathbf{u}) =
L_{T}(\mathbf{u})+ \sum_t (a_t-1) \ell_{t}(\mathbf{u})  \leq
L_{T}(\mathbf{u})+ S \sum_t (a_t-1)~.
\label{decomposition}
\end{equation} 
From the definition of $a_t$ we
have,
$a_{t}-1=a_{t}^{2}\mathbf{x}_{t}^{\top}\mathbf{A}_{t}^{-1}\mathbf{x}_{t}\leq\frac{b}{b-1}a_{t}\mathbf{x}_{t}^{\top}\mathbf{A}_{t}^{-1}\mathbf{x}_{t}$
(see \eqref{t6}). Finally, following similar steps to Forster~\cite{Forster}
we have,
$\sum_{t=1}^{T}a_{t}\mathbf{x}_{t}^{\top}\mathbf{A}_{t}^{-1}\mathbf{x}_{t}\leq\ln\left|\frac{1}{b}\mathbf{A}_{T}\right|$
(see \eqref{t7}).
\QED
\end{proofsketch}

Next we show a bound that may be sub-logarithmic if the comparison
vector $\vu$ suffers sub-linear amount of loss. Such a bound was
previously proposed by Orabona et.al~\cite{OrabonaCBG12}. We defer the
discussion about the bound after providing the proof below.
\begin{theorem}
\label{thm:theorem4}
Assume $\left\Vert \mathbf{x}_{t}\right\Vert \leq1$ for $t=1 \dots T$, 
and $b>1$. Assume further that 
\begin{equation}
a_{t}=\frac{1}{1-\mathbf{x}_{t}^{\top}\mathbf{A}_{t-1}^{-1}\mathbf{x}_{t}} \label{at}
\end{equation}
for $t=1 \dots T$. 
Then,
\begin{align}
L_{T}^{\boldsymbol{a}}(\mathbf{u})\leq L_{T}(\mathbf{u})+\frac{b}{b-1}Sd\left[1+\ln\left(1+\frac{L_T\paren{\vu}}{Sd}\right)\right]~.
\end{align}
\end{theorem}
We prove the theorem with a refined bound on the sum $\sum_t (a_t-1)
\ell_{t}(\mathbf{u})$ of \eqref{decomposition} using the
following two lemmas. 
In
\thmref{thm:theorem3} we bound the loss of all examples with $S$ and
then bound the remaining term. Here, instead we show a relation to a
subsequence ``pretending'' all examples of it as suffering a loss $S$, yet with the same cumulative loss, yielding
an effective shorter sequence, which we then bound. In the next lemma
we show how to find this subsequence, and in the following one bound
the performance.
\begin{lemma}
Let $I\subset \{1 \dots T\}$ be the indices of the $T^{'}=\left\lceil
  \sum_{t=1}^{T}\ell_{t}\left(\mathbf{u}\right)/S\right\rceil$ largest elements of $a_t$,
%\kc{is the last sum uses the loss of the algorithm of $\vu$?}
%\edward{the loss of  $\vu$}
that is $\vert I \vert = T'$ and $\min_{t \in I} a_t \geq a_\tau$ for
all $\tau \in \{ 1 \dots T\} / I$. Then,
\[
\sum_{t=1}^{T} \ell_{t}\left(\mathbf{u}\right) \left(a_{t}-1\right) \leq
S\sum_{t \in I}\left(a_{t}-1\right) ~.
\]
\label{lem:stacking_examples}
\end{lemma}
\begin{proof}
For a vector $\vv\in\reals^T$ define by $I(\vv)$ the set of indicies
of the $T^{'}$ maximal absolute-valued elements of $\vv$, and define $f(\vv) = \sum_{t \in
  I(\vv)} \vert \vvi{t} \vert$. The function $f(\vv)$ is a
norm~\cite{DekelLoSi07} with a dual norm $g(\vh) = \max\left\{ \Vert \vh
\Vert_\infty,\frac{\Vert \vh \Vert_1}{T^{'}}\right\}$.
From the property of dual norms we have $\vv\cdot\vh \leq f(\vv)
g(\vh)$. Applying this inequality to $\vv = (a_1-1 \comdots a_T-1)$ and
$\vh=(\ell_1\left(\mathbf{u}\right) \comdots \ell_T\left(\mathbf{u}\right))$ we get,
\[
\sum_{t=1}^{T} \ell_t \left(\mathbf{u}\right) (a_t-1) \leq \max\left\{
  S,\frac{\sum_{t=1}^{T}\ell_{t}\left(\mathbf{u}\right)}{T^{'}}\right\} \sum_{t \in I}
(a_t-1)~.
\] 
Combining with
$ST' 
= S \left\lceil      \sum_{t=1}^{T}\ell_{t}\left(\mathbf{u}\right)/S\right\rceil 
\geq
  \sum_{t=1}^{T}\ell_{t}\left(\mathbf{u}\right)$, completes the proof.
\QED
\end{proof}
% \begin{proof}
%   We assume wlog that
%   $\sum_{t=1}^{T}\ell_{t}/S$ is integral, as otherwise we may increase
%   the loss of one example artificially. We describe now an iterative
%   process. We initialize $\eta_t = \ell_{t}$ and repeat the following
%   steps. As long there exists $t \in I$ such that
%   $\eta_t < S$ by this invariance there exists $\tau \notin I$ with
%   $\eta_\tau>0$, we define $\delta = \min \{ S - \eta_t, \eta_\tau\}$
% and define $\eta_t \leftarrow \eta_t + \delta$ and $\eta_\tau
% \leftarrow \eta_\tau - \delta$. Clearly, both $\eta_t$ and $\eta_\tau$
% were modified (the former increased and the later decreased); and
% $\eta_t \leq S$  and $\eta_\tau \geq 0$. The process maintains an invariant $\sum_t \eta_t =
%   \sum_{t=1}^{T}\ell_{t}$ and by definition $ST' = S \left\lceil
%     \sum_{t=1}^{T}\ell_{t}/S\right\rceil \geq
%   \sum_{t=1}^{T}\ell_{t}$.
% \QED
% \end{proof}
Note that the quantity $\sum_{t \in I} a_t$ is based only on $T'$
examples, yet was generated using all $T$ examples. In fact by running
the algorithm with only these $T'$ examples the corresponding sum cannot get smaller. Specifically, assume the algorithm is run with inputs
$(\vxi{1},\yi{1}) \comdots  (\vxi{T},\yi{T})$ and generated a corresponding
sequence $(a_1 \comdots  a_T)$. Let $I$ be the set of indices with maximal
values of $a_t$ as before. Assume the algorithm is run with the subsequence of examples from $I$ (with the same order) and generated
$\alpha_1 \comdots  \alpha_T$ (where we set $\alpha_t=0$ for $t\notin I)$. Then, $\alpha_t \geq a_t$ for all $t\in I$.
This statement follows from \eqref{Adef} from which we get that the
matrix $\mathbf{A}_t$ is monotonically increasing in $t$. Thus, by removing
examples we get another smaller matrix which leads to a
larger value of $\alpha_t$.  

We continue the analysis with a sequence of length
$T'$ rather than a subsequence of the original sequence of length $T$ being analyzed.
The next lemma upper bounds the sum $\sum_t^{T'} a_t$ over $T'$
inputs with another sum of same length, yet using orthonormal set of
vectors of size $d$.
\begin{lemma}
  Let $\vxi{1} \comdots  \vxi{\tau}$ be any $\tau$ inputs with
  unit-norm. Assume the algorithm is performing updates using
  \eqref{at} for some $\mathbf{A}_0$ resulting in a sequence $a_1
  \comdots  a_\tau$. Let $E = \{ \vvi{1} \comdots  \vvi{d} \} \subset\reals^d$ be
  an eigen-decomposition of $\mathbf{A}_0$ with corresponding
  eigenvalues $\lambda_1 \comdots  \lambda_d$. Then there exists a
  sequence of indices $j_1 \comdots  j_\tau$, where $j_i \in \{ 1 \comdots  d\}$, such that
  $\sum_t a_t \leq \sum_t \alpha_t$, where $\alpha_t$ are generated
  using \eqref{at} on the sequence $\vvi{j_1}
  \comdots  \vvi{j_\tau}$. 

Additionally, let $n_s$ be the number of times
  eigenvector $\vvi{s}$ is used ($s=1 \comdots  d$), that is $n_s = \vert \{ j_t ~:~
  j_t=s\} \vert$ (and $\sum_s n_s = \tau$), then,
\[
\sum_t \alpha_t \leq \tau + \sum_{s=1}^d \sum_{r=1}^{n_s}
\frac{1}{\lambda_s+r-2} ~.
\]
\label{lem:worst_sequence}
\end{lemma}
\begin{proof}
By induction over $\tau$. For $\tau=1$ we want to upper bound $a_1 =
1/(1-\vxti{1} \mathbf{A}_0^{-1} \vxi{1})$ which is maximized when
$\vxi{1}=\vvi{d}$ the eigenvector with minimal eigenvalue $\lambda_d$,
in this case we have $\alpha_1 = 1/(1-1/\lambda_d) =
1+1/(\lambda_d-1)$, as desired.

Next we assume the lemma holds
for some $\tau-1$ and show it for $\tau$. Let $\vxi{1}$ be the first
input, and let $\{\gamma_s\}$ and $\{\vui{s}\}$ be the eigen-values and
eigen-vectors of
$\mathbf{A}_1 = \mathbf{A}_0 + a_1 \vxi{1}\vxti{1}$. The assumption of induction implies that $\sum_{t=2}^\tau \alpha_t \leq (\tau-1) + \sum_{s=1}^d \sum_{r=1}^{n_s}
\frac{1}{\gamma_s+r-2}$. From Theorem~8.1.8 of~\cite{Golub:1996:MC:248979} we know that the
eigenvalues of $\mathbf{A}_1$ satisfy $\gamma_s = \lambda_s + m_s$ for some $m_s\geq 0$ and
$\sum_s m_s =1$. We thus conclude that
\[
\sum_t a_t \leq 1+1/(\lambda_d-1) +  (\tau-1) + \sum_{s=1}^d \sum_{r=1}^{n_s}
\frac{1}{\lambda_s + m_s + r-2} ~.
\]
The last term is convex in $m_1 \comdots m_d$ and thus is maximized over a
vertex of the simplex, that is when $m_k=1$ for some $k$ and zero
otherwise. In this case, the eigen-vectors $\{\vui{s}\}$ of
$\mathbf{A}_1$ are in fact the eigenvectors $\{ \vvi{s} \}$
of $\mathbf{A}_0$, and the proof is completed.
\QED
\end{proof}
\begin{table}[t]
\begin{center}
\begin{tabular}{ll}\hline
Algorithm & Bound on Regret $R_T(\vu)$ \\\hline
Vovk~\cite{Vovk01} & ~$b\normt{\vu} + dY^2 \ln\paren{1+\frac{T}{db}}$\\
Forster~\cite{Forster} & ~$b\normt{\vu} + dY^2 \ln\paren{1+\frac{T}{db}}$\\
Crammer et.al.~\cite{CrammerKuDr12} & ~$rb\normt{\vu} + dA \ln\paren{1+\frac{T}{drb}}$\\
Orabona et.al.~\cite{OrabonaCBG12} & $ 2\normt{\vu} +
d(U+Y)^2\ln\paren{1+\frac{2\normt{\vu} + \sum_t
    \ell_t(\vu)}{d(U+Y)^2}}$\\
\thmref{thm:theorem3} & $b\normt{\vu} + Sd\frac{b}{b-1} \ln\left(1+\frac{T}{d(b-1)}\right)$\\
\thmref{thm:theorem4} & $b\normt{\vu} + Sd\frac{b}{b-1} \ln\paren{1+\frac{L_T(\vu)}{Sd}}$\\\hline
\end{tabular}
\end{center}
\caption{Comparison of regret bounds for online regression}
\label{tab:bounds}
\end{table}
Equipped with these lemmas we now prove \thmref{thm:theorem4}.
\begin{proof}
Let $T^{'}=\left\lceil
  \sum_{t=1}^{T}\ell_{t}/S\right\rceil$. Our starting point is the equality
\(
L_{T}^{\boldsymbol{a}}(\mathbf{u}) =
L_{T}(\mathbf{u})+ \sum_{t=1}^{T} \ell_{t}\left(\mathbf{u}\right)
\left(a_{t}-1\right)
\) 
stated in \eqref{decomposition}. From \lemref{lem:stacking_examples} 
we get,
\begin{equation}
 \sum_{t=1}^{T} \ell_{t}\left(\mathbf{u}\right)
\left(a_{t}-1\right)
 \leq
S\sum_{t \in I}\left(a_{t}-1\right) \leq
S\sum_{t}^{T'}\left(\alpha_{t}-1\right) ~,
\label{chain1}
\end{equation}
where $I$ is the subset of $T'$ indices for which $a_t$ are maximal,
and $\alpha_t$ are the resulting coefficients computed with \eqref{at}
using only the
sub-sequence of examples $\vxi{t}$ with $t\in I$.

By definition $\mathbf{A}_0 = b \mi$ and thus from
\lemref{lem:worst_sequence} we further bound \eqref{chain1} with,
\begin{equation}
 \sum_{t=1}^{T} \ell_{t}\left(\mathbf{u}\right)
\left(a_{t}-1\right)
 \leq
S \sum_{s=1}^d \sum_{r=1}^{n_s}
\frac{1}{b+r-2} ~,
\label{chain11}
\end{equation}
for some $n_s$ such that $\sum_s n_s = T'$.
The last equation is maximized when all the counts $n_s$ are about (as
$d$ may not divide $T'$) the
same, and thus we further bound \eqref{chain11} with,
\begin{align*}
\sum_{t=1}^{T} \ell_{t}\left(\mathbf{u}\right)
\left(a_{t}-1\right)
 & \leq~ S \sum_{s=1}^d \sum_{r=1}^{\lceil T'/d \rceil}
\frac{1}{b+r-2}  \leq Sd \sum_{r=1}^{\lceil T'/d \rceil}
\frac{b}{b-1}\frac{1}{r}\\
 &\leq~ Sd \frac{b}{b-1} \paren{1+\ln\paren{\left\lceil\frac{T'}{d}\right\rceil}} \\
& \leq~ Sd \frac{b}{b-1} \paren{1+\ln\paren{1+\frac{L_T\paren{\vu}}{Sd}}} ~,
\end{align*}
which completes the proof.
\QED
\end{proof}

It is instructive to compare bounds of similar algorithms, summarized
in \tabref{tab:bounds}. Our first bound\footnote{The bound in the table is obtained by noting that $\log\det$ is a
concave function of the eigenvalues of the matrix, upper bounded when
all the eigenvalues are equal (with the same trace).} of \thmref{thm:theorem3} is most similar to the bounds
of Forster~\cite{Forster}, Vovk~\cite{Vovk01} and Crammer et.al.~\cite{CrammerKuDr12}. 
Forster and Vovk have a multiplicative factor $Y^2$ of the logarithm, Crammer et.al. have the factor $A= \sup_{1\leq t\leq T}\ell_{t}(\textrm{alg})$, and we have the
worst-loss of $\vu$ over all examples (denoted by $S$). Thus, our first bound is better than the bound of Crammer et.al. (as often $S<A$), and better than the bounds of Forster and Vovk 
on problems that are approximately linear $y_{t}
\approx\mathbf{u}\cdot\mathbf{x}_{t}$ for $t = 1 \comdots T$ and $Y$
is large, while their bound is better if $Y$ is small. Note that the analysis of Forster~\cite{Forster} assumes that the
labels $\yi{t}$ are bounded, and formally the algorithm should know
this bound, while Crammer et.al. assume that the inputs are
bounded, as we do.

Our second bound of \thmref{thm:theorem4} is similar to the bound of
Orabona et.al.~\cite{OrabonaCBG12}. Both bounds have potentially
sub-logarithmic regret as the cumulative loss $L(\vu)$ may be
sublinear in $T$. Yet, their bound has a multiplicative factor of
$(U+Y)^2$, while our bound has only the maximal loss $S$,
which, as before, can be much smaller. Additionally, their
analysis assumes that both the inputs $\vxi{t}$ and the labels
$\yi{t}$ are bounded, while we only assume that the inputs are bounded, and
furthermore, our algorithm does not need to assume and know a compact
set which contains $\vu$ ($\left\Vert\vu\right\Vert\leq U$), as opposed to their algorithm.

\section{Learning in Non-Stationary Environment}
\label{sec:non_stat}
In this section we present a generalization of the last-step
min-max predictor for non-stationary problems given in
\eqref{minmax_algorithm_1}.
We define the predictor to be,
\begin{align}
\hyi{T} = \arg\min_{\hyi{T}} \max_{\yi{T}} \brackets{\sum_{t=1}^{T} (\yi{t} -
  \hyi{t})^2  - 
\inf_{\mathbf{u}_{1},\ldots,\mathbf{u}_{T},\bar{\mathbf{u}}}  \left(b\left\Vert
    \bar{\mathbf{u}}\right\Vert ^{2}+c V_m+L_{T}^{\widetilde{\boldsymbol{a}}}(\mathbf{u}_{1},\ldots,\mathbf{u}_{T})\right)}
\label{minNoStatEq}
% \inf_{\vu} \paren{b\left\Vert \mathbf{u}\right\Vert
%     ^{2}+L_{T}^{\boldsymbol{a}}(\vu)}}~,
% \label{minmax_algorithm_2}
\end{align}
for
\begin{align}
V_m=\sum_{t=1}^{T}\left\Vert
    \mathbf{u}_{t}-\bar{\mathbf{u}}\right\Vert
  ^{2}~,\label{V_m}
\end{align}
positive constants $b,c>0$ and weights
$\widetilde{a}_{t}\geq1$ for $1\leq t\leq T$.

As mentioned above, we use an extended notion of function class,
using different vectors $\vui{t}$ across time $T$. We circumvent here
the problem mentioned in the end of \secref{sec:problem_setting}, and
restrict the adversary from choosing an arbitrary $T$-tuple $(\vui{1}
\comdots \vui{T})$ by introducing a reference weight-vector
$\bar{\vu}$. Specifically, indeed we replace the single-weight
cumulative-loss $L_{T}^{{\boldsymbol{a}}}(\mathbf{u})$ in \eqref{minmax_algorithm_1}  with a
multi-weight cumulative-loss $L_{T}^{\widetilde{\boldsymbol{a}}}(\mathbf{u}_{1},\ldots,\mathbf{u}_{T})$
in \eqref{minNoStatEq}, yet, we add the term $c V_m$ to \eqref{minNoStatEq} penalizing a $T$-tuple $(\vui{1}
  \comdots \vui{T})$ that its elements $\{\vui{t}\}$ are far from some
  single point $\bar{\vu}$. Intuitively, $V_m$ serves as a measure of
  complexity of the $T$-tuple by measuring the deviation of its elements
  from some vector.

The new formulation of \eqref{minNoStatEq} clearly subsumes the
formulation of \eqref{minmax_algorithm_1}, as if
$\vui{1}=\dots\vui{T}=\bar{\vu}=\vu$, then \eqref{minNoStatEq} reduces
to \eqref{minmax_algorithm_1}. We now show that in-fact the two
notions of last-step min-max predictors are equivalent. The following
lemma characterizes the solution of the inner infimum of
\eqref{minNoStatEq} over $\bar{\vu}$.
\begin{lemma}
\label{lem:lemma4}
For any $\bar{\mathbf{u}}\in\mathbb{R}^{d}$, the function 
\[
J\left(\mathbf{u}_{1},\ldots,\mathbf{u}_{T}\right)=b\left\Vert
  \bar{\mathbf{u}}\right\Vert ^{2}+c\sum_{t=1}^{T}\left\Vert
  \mathbf{u}_{t}-\bar{\mathbf{u}}\right\Vert
^{2}+\sum_{t=1}^{T}\widetilde{a}_{t}\left(y_{t}-\mathbf{u}_{t}^{\top}\mathbf{x}_{t}\right)^{2}~,
\]
is minimal for 
\[
\mathbf{u}_{t}=\bar{\mathbf{u}}+\frac{c^{-1}}{\widetilde{a}_{t}^{-1}+c^{-1}\left\Vert
    x_{t}\right\Vert
  ^{2}}\left(y_{t}-\bar{\mathbf{u}}^{\top}\mathbf{x}_{t}\right)\mathbf{x}_{t}
\]
for $t=1 ... T$. 
The minimal value of \textup{$J\left(\mathbf{u}_{1},\ldots,\mathbf{u}_{T}\right)$
}is given by 
\begin{align}
J_{min}=b\left\Vert \bar{\mathbf{u}}\right\Vert
^{2}+\sum_{t=1}^{T}\frac{1}{\widetilde{a}_{t}^{-1}+c^{-1}\left\Vert
    \mathbf{x}_{t}\right\Vert
  ^{2}}\left(y_{t}-\bar{\mathbf{u}}^{\top}\mathbf{x}_{t}\right)^{2} ~.
\label{optimal_J}
\end{align}
\end{lemma}
The proof appears in \ref{proof_lemma4}.
%\kc{Add MAP interoperation for both algorithms}
%\edward{Done.}
%
\begin{Remark}
\label{MAP2}
The minimization problem in \lemref{lem:lemma4} can be interpreted as MAP estimator
of $\bar{\mathbf{u}}$ based on the sequence $\left\{ \left(\mathbf{x}_{t},y_{t}\right)\right\} _{t=1}^{T}$
in the following generative model:
\begin{eqnarray*}
\bar{\mathbf{u}} & \sim & N\left(0,\sigma_{b}^{2}\mathbf{I}\right)\\
\mathbf{u}_{t} & \sim & N\left(\bar{\mathbf{u}},\sigma_{c}^{2}\mathbf{I}\right)\\
y_{t} & \sim & N\left(\mathbf{x}_{t}^{\top}\mathbf{u}_{t},\sigma_{t}^{2}\right)~,
\end{eqnarray*}
where $\sigma_{b}^{2}=\frac{1}{2b}$, $\sigma_{c}^{2}=\frac{1}{2c}$ and
$\sigma_{t}^{2}=\frac{1}{2\widetilde{a}_{t}}$.

Indeed, 
\begin{eqnarray}
\bar{\mathbf{u}}_{MAP} & = & \arg\max_{\bar{\mathbf{u}}}P\left(\bar{\mathbf{u}}\mid\left\{ \mathbf{u}_{t}\right\} ,\left\{ \mathbf{x}_{t}\right\} ,\left\{ y_{t}\right\} \right) \nonumber \\
 & = & \arg\max_{\bar{\mathbf{u}}}\left[P\left(\bar{\mathbf{u}}\right)\prod_{t=1}^{T}P\left(\mathbf{u}_{t}\mid\bar{\mathbf{u}}\right)\prod_{t=1}^{T}P\left(y_{t}\mid\mathbf{u}_{t},\mathbf{x}_{t}\right)\right] \nonumber \\
 & = & \arg\min_{\bar{\mathbf{u}}}\left[-\log P\left(\bar{\mathbf{u}}\right)-\sum_{t=1}^{T}\log P\left(\mathbf{u}_{t}\mid\bar{\mathbf{u}}\right)-\sum_{t=1}^{T}\log P\left(y_{t}\mid\mathbf{u}_{t},\mathbf{x}_{t}\right)\right]~. \label{u_bar_map}
\end{eqnarray}
By our gaussian generative model we have
\begin{align*}
&-\log P\left(\bar{\mathbf{u}}\right)&=&\log\left(2\pi\sigma_{b}^{2}\right)^{d/2}+\frac{1}{2\sigma_{b}^{2}}\left\Vert \bar{\mathbf{u}}\right\Vert ^{2}\\
&-\log P\left(\mathbf{u}_{t}\mid\bar{\mathbf{u}}\right)&=&\log\left(2\pi\sigma_{c}^{2}\right)^{d/2}+\frac{1}{2\sigma_{c}^{2}}\left\Vert \mathbf{u}_{t}-\bar{\mathbf{u}}\right\Vert ^{2}\\
&-\log P\left(y_{t}\mid\mathbf{u}_{t},\mathbf{x}_{t}\right)&=&\log\left(2\pi\sigma_{t}^{2}\right)^{1/2}+\frac{1}{2\sigma_{t}^{2}}\left(y_{t}-\mathbf{x}_{t}^{\top}\mathbf{u}_{t}\right)^{2}~.
\end{align*}
Substituting in \eqref{u_bar_map} we get
\[
 \bar{\mathbf{u}}_{MAP}=\arg\min_{\bar{\mathbf{u}}}\left[\frac{1}{2\sigma_{b}^{2}}\left\Vert \bar{\mathbf{u}}\right\Vert ^{2}+\frac{1}{2\sigma_{c}^{2}}\sum_{t=1}^{T}\left\Vert \mathbf{u}_{t}-\bar{\mathbf{u}}\right\Vert ^{2}+\sum_{t=1}^{T}\frac{1}{2\sigma_{t}^{2}}\left(y_{t}-\mathbf{x}_{t}^{\top}\mathbf{u}_{t}\right)^{2}\right]~,
\]
and by using $\frac{1}{2\sigma_{b}^{2}}=b$, $\frac{1}{2\sigma_{c}^{2}}=c$,
$\frac{1}{2\sigma_{t}^{2}}=\widetilde{a}_{t}$ we get the minimization
problem in \lemref{lem:lemma4}. 
\end{Remark}

Substituting \eqref{optimal_J} in \eqref{minNoStatEq} we obtain the
following form of the last-step minmax predictor,
\begin{align}
\hyi{T}=\arg\min_{\hyi{T}} \max_{\yi{T}} \brackets{\sum_{t=1}^{T} (\yi{t} -
  \hyi{t})^2   
-\inf_{\bar{\mathbf{u}}\in\mathbb{R}^{d}}\left(b\left\Vert
    \bar{\mathbf{u}}\right\Vert
  ^{2}+\sum_{t=1}^{T}\frac{1}{\widetilde{a}_{t}^{-1}+c^{-1}\left\Vert
      \mathbf{x}_{t}\right\Vert
    ^{2}}\left(y_{t}-\mathbf{x}_{t}^{\top}\bar{\mathbf{u}}\right)^{2}\right)}~.
% \inf_{\mathbf{u}_{1},\ldots,\mathbf{u}_{T},\bar{\mathbf{u}}}\left(b\left\Vert
%     \bar{\mathbf{u}}\right\Vert ^{2}+c\sum_{t=1}^{T}\left\Vert
%     \mathbf{u}_{t}-\bar{\mathbf{u}}\right\Vert
%   ^{2}+L_{T}^{\widetilde{\boldsymbol{a}}}(\mathbf{u}_{1},\ldots,\mathbf{u}_{T})\right)}
\label{minNoStatEq2}
\end{align}
% \[
% L_{T}(\texttt{WEMM})-\min_{\bar{\mathbf{u}}\in\mathbb{R}^{d}}\left(b\left\Vert \bar{\mathbf{u}}\right\Vert ^{2}+\sum_{t=1}^{T}\frac{1}{\widetilde{a}_{t}^{-1}+c^{-1}\left\Vert \mathbf{x}_{t}\right\Vert ^{2}}\left(y_{t}-\mathbf{x}_{t}^{\top}\bar{\mathbf{u}}\right)^{2}\right)
% \]
Clearly, both equations \eqref{minNoStatEq2} and
\eqref{minmax_algorithm_1} are equivalent when identifying,
\begin{equation}
a_{t}=\frac{1}{\widetilde{a}_{t}^{-1}+c^{-1}\left\Vert
    \mathbf{x}_{t}\right\Vert ^{2}}  ~.\label{aa}
\end{equation}
% and note that the minimization of 
% \[
% L_{T}(\texttt{WEMM})-\min_{\bar{\mathbf{u}}\in\mathbb{R}^{d}}\left(b\left\Vert \bar{\mathbf{u}}\right\Vert ^{2}+\sum_{t=1}^{T}a_{t}\left(y_{t}-\mathbf{x}_{t}^{\top}\bar{\mathbf{u}}\right)^{2}\right)
% \]
% is equivalent to single expert model in section 2. We use the results
% in section 2 to get the next conclusions:
Therefore, we can use the results of the previous
sections. 
%Specifically, if 
%\(
%1+a_{T}\mathbf{x}_{T}^{\top}\mathbf{A}_{T-1}^{-1}\mathbf{x}_{T}-a_{T}\leq
%0
%\) 
%the optimal predictor developed in \thmref{thm:theorem1} for the stationary case
%is given by, % (see \eqref{laststep_minmax_optimal})
\begin{corollary}
The optimal prediction for the last round $T$ is
\(
\hat{y}_{T}=\mathbf{b}_{T-1}^{\top}\mathbf{A}_{T-1}^{-1}\mathbf{x}_{T}
\)
  if the following condition is hold
 \[
 1+a_{T}\mathbf{x}_{T}^{\top}\mathbf{A}_{T-1}^{-1}\mathbf{x}_{T}-a_{T}\leq0 ~,
 \]
 where $a_{T}$ defined by \eqref{aa} and 
where we replace \eqref{Adef} with 
\begin{eqnarray*}
\mathbf{A}_{t}  = 
b\mathbf{I}+\sum_{s=1}^{t}a_{s}\mathbf{x}_{s}\mathbf{x}_{s}^{\top}=b\mathbf{I}+\sum_{s=1}^{t}\frac{1}{\widetilde{a}_{s}^{-1}+c^{-1}\left\Vert
    \mathbf{x}_{s}\right\Vert
  ^{2}}\mathbf{x}_{s}\mathbf{x}_{s}^{\top}
\end{eqnarray*}
and \eqref{bdef} with, 
\begin{eqnarray*}
\mathbf{b}_{t}  =  \sum_{s=1}^{t}a_{s}y_{s}\mathbf{x}_{s}=\sum_{s=1}^{t}\frac{1}{\widetilde{a}_{s}^{-1}+c^{-1}\left\Vert \mathbf{x}_{s}\right\Vert ^{2}}y_{s}\mathbf{x}_{s}.
\end{eqnarray*}
\end{corollary}
Although most of the analysis above holds for $1+a_{t}\mathbf{x}_{t}^{\top}\mathbf{A}_{t-1}^{-1}\mathbf{x}_{t}-a_{t}\leq
0$ in the end of the day, \thmref{thm:theorem3} assumed that this inequality holds as equality. Substituting  \(
a_{t}=\frac{1}{1-\mathbf{x}_{t}^{\top}\mathbf{A}_{t-1}^{-1}\mathbf{x}_{t}}
\)
in \eqref{aa} and solving for $\widetilde{a}_{t}$ we obtain,
% \[
% \frac{1}{\widetilde{a}_{t}^{-1}+c^{-1}\left\Vert \mathbf{x}_{t}\right\Vert ^{2}}=\frac{1}{1-\mathbf{x}_{t}^{\top}\mathbf{A}_{t-1}^{-1}\mathbf{x}_{t}}
% \]
% \[
% \widetilde{a}_{t}^{-1}+c^{-1}\left\Vert \mathbf{x}_{t}\right\Vert ^{2}=1-\mathbf{x}_{t}^{\top}\mathbf{A}_{t-1}^{-1}\mathbf{x}_{t}
% \]
\begin{align}
\widetilde{a}_{t}=\frac{1}{1-\mathbf{x}_{t}^{\top}\mathbf{A}_{t-1}^{-1}\mathbf{x}_{t}-c^{-1}\left\Vert
    \mathbf{x}_{t}\right\Vert ^{2}} ~. \label{def_tilde_a}
\end{align}
The last-step minmax predictor \eqref{minNoStatEq}
is convex if $\widetilde{a}_{t} \geq 0$, which holds if,
\(%\begin{align}
1/b +
1/c \leq 1~, %\label{relate_b_c}
\)%\end{align}
because $\mathbf{A}_{t-1}^{-1} \preceq \mathbf{A}_{0}^{-1} = (1/b)\mi$
and we assume that $\normt{\vxi{t}}\leq 1$.
  
Let us state the analogous statements of  \thmref{thm:theorem2} and
\thmref{thm:theorem3}. Substituting \lemref{lem:lemma4} in
\thmref{thm:theorem2} we bound the cumulative loss of the algorithm
with the weighted loss of any $T$-tuple $(\vui{1} \comdots \vui{T})$.
\begin{corollary}
\label{cor:first_nonstationary_bound_1}
  Assume $\left\Vert \mathbf{x}_{t}\right\Vert \leq1$, 
  $1+a_{t}\mathbf{x}_{t}^{\top}\mathbf{A}_{t-1}^{-1}\mathbf{x}_{t}-a_{t}\leq
  0$ for $t =1 \dots T$, and $1/b + 1/c \leq 1$. Then, the loss of the
  last-step minmax predictor,
  $\hat{y}_{t}=\mathbf{b}_{t-1}^{\top}\mathbf{A}_{t-1}^{-1}\mathbf{x}_{t}$
  for $t =1 \dots T$, is upper bounded by,
\[
L_{T}(\texttt{WEMM})\leq\inf_{\mathbf{u}\in\mathbb{R}^{d}}\left(b\left\Vert
    \mathbf{u}\right\Vert
  ^{2}+L_{T}^{\boldsymbol{a}}(\mathbf{u})\right) 
=\inf_{\mathbf{u}_{1},\ldots,\mathbf{u}_{T},\bar{\mathbf{u}}}\left(b\left\Vert
    \bar{\mathbf{u}}\right\Vert ^{2}+c\sum_{t=1}^{T}\left\Vert
    \mathbf{u}_{t}-\bar{\mathbf{u}}\right\Vert
  ^{2}+L_{T}^{\widetilde{\boldsymbol{a}}}(\mathbf{u}_{1},\ldots,\mathbf{u}_{T})\right) ~.
\] 
Furthermore, if
$1+a_{t}\mathbf{x}_{t}^{\top}\mathbf{A}_{t-1}^{-1}\mathbf{x}_{t}-a_{t}
= 0$, then the last inequality is in fact an equality.
% If learner predicts with $\hat{y}_{t}=\mathbf{b}_{t-1}^{\top}\mathbf{A}_{t-1}^{-1}\mathbf{x}_{t}$
% for $t =1 \dots T$ and if $1+a_{t}\mathbf{x}_{t}^{\top}\mathbf{A}_{t-1}^{-1}\mathbf{x}_{t}-a_{t}\leq0$
% for $t =1 \dots T$ then
% \[
% L_{T}(\texttt{WEMM})\leq\min_{\bar{\mathbf{u}}\in\mathbb{R}^{d}}\left(b\left\Vert \bar{\mathbf{u}}\right\Vert ^{2}+L_{T}^{\boldsymbol{a}}(\bar{\mathbf{u}})\right)=\min_{\mathbf{u}_{1},\ldots,\mathbf{u}_{T},\bar{\mathbf{u}}}\left(b\left\Vert \bar{\mathbf{u}}\right\Vert ^{2}+c\sum_{t=1}^{T}\left\Vert \mathbf{u}_{t}-\bar{\mathbf{u}}\right\Vert ^{2}+L_{T}^{\widetilde{\boldsymbol{a}}}(\mathbf{u}_{1},\ldots,\mathbf{u}_{T})\right)
% \]
\end{corollary}
Next we relate the weighted cumulative loss
$L_{T}^{\widetilde{\boldsymbol{a}}}(\mathbf{u}_{1},\ldots,\mathbf{u}_{T})$
to the loss itself $L_{T}(\mathbf{u}_{1},\ldots,\mathbf{u}_{T})$,
\begin{corollary}
\label{cor:loss_a_tilde_loss}
Assume $\left\Vert \mathbf{x}_{t}\right\Vert \leq1$ for $t=1 \dots T$, $b>1$ and $1/b + 1/c \leq 1$. 
Assume additionally that 
$\widetilde{a}_{t}=\frac{1}{1-\mathbf{x}_{t}^{\top}\mathbf{A}_{t-1}^{-1}\mathbf{x}_{t}-c^{-1}\left\Vert
    \mathbf{x}_{t}\right\Vert ^{2}}$ as given in \eqref{def_tilde_a}. 
Then 
\[
L_{T}^{\widetilde{\boldsymbol{a}}}(\mathbf{u}_{1},\ldots,\mathbf{u}_{T}) 
\leq 
L_{T}(\mathbf{u}_{1},\ldots,\mathbf{u}_{T}) 
+\frac{b}{b-1}S\ln\left|\frac{1}{b}\mathbf{A}_{T}\right|
+TS\frac{1}{c\left(1-b^{-1}\right)^{2}-\left(1-b^{-1}\right)}~.
\]
\end{corollary}
\begin{proof}
 We start as in the proof of \thmref{thm:theorem3} and decompose the
  weighted loss, 
\begin{align}
L_{T}^{\widetilde{\boldsymbol{a}}}(\mathbf{u}_{1},\ldots,\mathbf{u}_{T}) 
&=
  L_{T}(\mathbf{u}_{1},\ldots,\mathbf{u}_{T}) + \sum_t
  (\widetilde{a}_t-1) \ell_{t}(\mathbf{u}_t) \nonumber\\
%\leq  L_{T}(\mathbf{u}_{1},\ldots,\mathbf{u}_{T})+ S \sum_t
%(\widetilde{a}_t-1) 
&\leq 
 L_{T}(\mathbf{u}_{1},\ldots,\mathbf{u}_{T})+ S \sum_t
 (a_t-1) +S \sum_t (\widetilde{a}_t-a_t) ~. \label{basic_bound}
\end{align}
We bound the sum of the third term,
\begin{align}
\widetilde{a}_{t}-a_{t} 
& =  \frac{1}{1-\mathbf{x}_{t}^{\top}\mathbf{A}_{t-1}^{-1}\mathbf{x}_{t}-c^{-1}\left\Vert \mathbf{x}_{t}\right\Vert ^{2}}-\frac{1}{1-\mathbf{x}_{t}^{\top}\mathbf{A}_{t-1}^{-1}\mathbf{x}_{t}}\nonumber\\
 & =  \frac{c^{-1}\left\Vert \mathbf{x}_{t}\right\Vert ^{2}}{\left(1-\mathbf{x}_{t}^{\top}\mathbf{A}_{t-1}^{-1}\mathbf{x}_{t}-c^{-1}\left\Vert \mathbf{x}_{t}\right\Vert ^{2}\right)\left(1-\mathbf{x}_{t}^{\top}\mathbf{A}_{t-1}^{-1}\mathbf{x}_{t}\right)}\nonumber\\
% \end{align}
% Same as in theorem 3, for $\left\Vert \mathbf{x}_{t}\right\Vert \leq1$
% we have $\mathbf{x}_{t}^{\top}\mathbf{A}_{t-1}^{-1}\mathbf{x}_{t}<b^{-1}$
% and therefore
% \begin{eqnarray}
%\widetilde{a}_{t}-a_{t} & < & 
&\leq
\frac{c^{-1}}{\left(1-b^{-1}-c^{-1}\right)\left(1-b^{-1}\right)}
  =  \frac{1}{c\left(1-b^{-1}\right)^{2}-\left(1-b^{-1}\right)} ~.\label{t10}
\end{align}
Additionally, as in \thmref{thm:theorem3} the second term is bounded with $
\frac{b}{b-1}S\ln\left|\frac{1}{b}\mathbf{A}_{T}\right|$. Substituting
this bound and \eqref{t10} in \eqref{basic_bound} completes the proof.
\QED
\end{proof}
Combining the last two corollaries yields the main result of this section.
\begin{corollary}
\label{cor:main_cor_non_stat_1}
Under the conditions of \corref{cor:loss_a_tilde_loss} the cumulative loss of the last-step
minmax predictor is upper bounded by,
\begin{align*}
L_{T}(\texttt{WEMM})
\leq&\inf_{\mathbf{u}_{1},\ldots,\mathbf{u}_{T},\bar{\mathbf{u}}}
\Bigg(
b\left\Vert
    \bar{\mathbf{u}}\right\Vert ^{2}+c V_m+
L_{T}(\mathbf{u}_{1},\ldots,\mathbf{u}_{T}) 
+\frac{Sb}{b-1}\ln\left|\frac{1}{b}\mathbf{A}_{T}\right|
+\frac{TS}{c\left(1-b^{-1}\right)^{2}-\left(1-b^{-1}\right)}\Bigg)~,
\end{align*}
where $V_m$ is the deviation of $\{\vui{t}\}$ from some fixed
weight-vector as defined in \eqref{V_m}. 
Additionally, setting
\(
c_V=\frac{b}{b-1}\left(1+\sqrt{\frac{ST}{V_m}}\right)
\) minimizing the above bound over $c$,
\begin{align*}
L_{T}(\texttt{WEMM})
\leq&\inf_{\mathbf{u}_{1},\ldots,\mathbf{u}_{T},\bar{\mathbf{u}}}
\Bigg(
b\left\Vert
    \bar{\mathbf{u}}\right\Vert ^{2}+
L_{T}(\mathbf{u}_{1},\ldots,\mathbf{u}_{T}) 
+\frac{Sb}{b-1}\ln\left|\frac{1}{b}\mathbf{A}_{T}\right|
+\frac{b}{b-1}\left(V_m+2\sqrt{STV_m}\right)\Bigg)~.
\end{align*}
\end{corollary}
Few comments. First, it is straightforward to verify that
$c_V=\frac{b}{b-1}\left(1+\sqrt{\frac{ST}{V_m}}\right)$ satisfy the
constraint $1/b+1/c_V \leq 1$. Second, this bound strictly generalizes
the bound for the stationary case, since
\corref{cor:loss_a_tilde_loss} reduces to \thmref{thm:theorem3} when
all the weight-vectors equal each other
$\vui{1}=\dots\vui{T}=\bar{\vu}$ (i.e. $V_m=0$). Third, the
constant $c$ (or $c_V$) is not used by the algorithm, but only in the
analysis. So there is no need to know the actual deviation $V_m$ to
tune the algorithm. In other words, the bound applies essentially to
the same last step minmax predictor defined in
\thmref{thm:theorem1}. Finally, we have a bound for the
non-stationary case based on \thmref{thm:theorem4} instead of
\thmref{thm:theorem3}, by replacing the term 
\[
\frac{Sb}{b-1}\ln\left|\frac{1}{b}\mathbf{A}_{T}\right| ~,
\] 
with 
\[
\frac{Sbd}{b-1} \paren{1+\ln\paren{1+\frac{\sum_{t}{\ell_t(\vu_t)}}{Sd}}}~.
\]

\section{Related work}
The problem of predicting reals in an online manner was studied for more than five decades. Clearly we cannot cover all previous work here, and the reader
is refered to the encyclopedic book of Cesa-Bianchi and Lugosi~\cite{CesaBiGa06}
 for a full survey.

Widrow and Hoff~\cite{WidrowHoff} studied a gradient descent algorithm for the squared loss. Many variants of the algorithm were studied since then. A notable example is the normalized least mean squares algorithm (NLMS)~\cite{Bitmead,Bershad} that adapts to the input's scale. More gradient descent based algorithms and bounds for regression with the squared loss were proposed by Cesa-Bianchi et.al.~\cite{Nicolo_Warmuth} about two decades ago. These algorithms were generalized and extended by Kivinen and Warmuth~\cite{Kiv_War} using additional regularization functions.

An online version of the ridge regression algorithm in the worst-case
setting was proposed and analyzed by Foster~\cite{Foster91}. A related
algorithm called the Aggregating Algorithm (AA) was studied by
Vovk~\cite{vovkAS}. See also the work of Azoury and
Warmuth~\cite{AzouryWa01}.

The recursive least squares (RLS)~\cite{Hayes} is a
similar algorithm proposed for adaptive filtering. A variant of the RLS algorithm (AROW for regression~\cite{VaitsCr11}) was analysed by Crammer et.al.~\cite{CrammerKuDr12}. All algorithms
make use of second order information, as they maintain a weight-vector
and a covariance-like positive semi-definite (PSD) matrix used to
re-weight the input. The eigenvalues of this covariance-like matrix
grow with time $t$, a property which is used to prove logarithmic
regret bounds. Orabona et.al.~\cite{OrabonaCBG12} showed
that beyond logarithmic regret bound can be achieved when the total
best linear model loss is sublinear in $T$. We derive a similar bound,
with a multiplicative factor that depends on the worst-loss of $\vu$, rather than a bound $Y$ on the labels. 
Hazan and Kale~\cite{HazanK09a} developed regret
bounds that depend logarithmically on the variance of the
side information used to define the loss sequence. In the
regression case, this corresponds to a bound that depends
on the variance of the instance vectors $\mathbf{x}_t$, rather than on
the loss of the competitor, as the bound of Orabona et.al.~\cite{OrabonaCBG12} and our bound. 

The derivation of our algorithm shares similarities with the work of
Forster~\cite{Forster}. Both algorithms are motivated from the
last-step min-max predictor. Yet, the formulation of
Forster~\cite{Forster} yields a convex optimization for which the max
operation over $\yi{t}$ is not bounded, and thus he used an
artificial clipping operation to avoid unbounded solutions. With a
proper tuning of $a_t$ and a weighted loss, we are able to obtain a
problem that is convex in $\hyi{t}$ and concave in $\yi{t}$, and thus
well defined.

Most recent work is focused in the stationary setting. We also discuss
a specific weak-notion of non-stationary setting, for which the few
weight-vectors can be used for comparison and their total deviation is
computed with respect to some single weight-vector. Recently, Vaits
and Crammer~\cite{VaitsCr11} proposed an algorithm designed for
non-stationary environments. Herbster and Warmuth \cite{HerbsterW01}
discussed general gradient descent algorithms with projection of the
weight-vector using the Bregman divergence, and
Zinkevich~\cite{Zinkevich03} developed an algorithm for online convex
programming. Busuttil and Kalnishkan~\cite{BusuttilK07} developed a variant of the aggregating algorithm in the non-stationary environment. 
They all use a stronger notion of diversity between
vectors, as their distance is measured with consecutive vectors (that
is drift that may end far from the starting point). Thus, the bounds in these papers cannot be compared in general to our bound in \corref{cor:main_cor_non_stat_1}.  
The $H_\infty$ filters~(see e.g. papers by Simon~\cite{Simon:2006:OSE:1146304,DBLP:journals/tsp/Simon06})
 are a family of (robust) linear filters developed based on a min-max approach, like \texttt{WEMM}, and analyzed in the worst
case setting. These filters are reminiscent of the celebrated Kalman
filter~\cite{Kalman60}, which was motivated and analyzed in a stochastic
setting with Gaussian noise.
Finally, few second-order algorithms were recently proposed in other contexts~\cite{CesaBianchiCoGe05,CrammerKuDr09,DuchiHS10,McMahanS10}.
%\kc{please add MORE references, also including H infinity, and more
%recent papers}
%\edward{Done.}

\section{Summary and Conclusions}

We proposed a modification of the last-step min-max algorithm~\cite{Forster} using weights over examples, and showed how to choose these weights for the problem to be well defined -- convex -- which enabled us to develop the last step min-max predictor, without requiring the labels to be bounded. Our algorithmic formulations depend on inner- and outer-products and thus can be employed with kernel functions. Our analysis bounds the regret with quantities that depend only on the loss of the competitor, with no need for any knowledge of the problem.  Our prediction algorithm was motivated from the last-step minmax predictor problem for stationary setting, but we showed that the same algorithm can be used to derive a bound for a class of {\em non-stationary} problems as well.

%An empirical study we conducted shows that our algorithm performs close to other algorithms~\cite{Hayes,Vovk01,Forster,CrammerKuDr12}.
An interesting direction would be to extend the algorithm for general loss functions rather than the squared loss, or to classification tasks.

\section*{Acknowledgements}  
The research is partially supported by an Israeli Science Foundation
grant ISF- 1567/10.

\appendix
\section{Proof of \lemref{lem:lemma2}}
\label{proof_lemma2}
\begin{proof}
Using the Woodbury identity we get
\[
\mathbf{A}_{t}^{-1}=\mathbf{A}_{t-1}^{-1}-\frac{\mathbf{A}_{t-1}^{-1}\mathbf{x}_{t}\mathbf{x}_{t}^{\top}\mathbf{A}_{t-1}^{-1}}{\frac{1}{a_{t}}+\mathbf{x}_{t}^{\top}\mathbf{A}_{t-1}^{-1}\mathbf{x}_{t}}~,
\]
 therefore the left side of \eqref{lemma2} is
\begin{eqnarray*}
a_{t}^{2}\mathbf{x}_{t}^{\top}\mathbf{A}_{t}^{-1}\mathbf{x}_{t}+1-a_{t} 
 &=&  a_{t}^{2}\mathbf{x}_{t}^{\top}\left(\mathbf{A}_{t-1}^{-1}-\frac{\mathbf{A}_{t-1}^{-1}\mathbf{x}_{t}\mathbf{x}_{t}^{\top}\mathbf{A}_{t-1}^{-1}}{\frac{1}{a_{t}}+\mathbf{x}_{t}^{\top}\mathbf{A}_{t-1}^{-1}\mathbf{x}_{t}}\right)\mathbf{x}_{t}+1-a_{t}\\
  &=&  a_{t}^{2}\mathbf{x}_{t}^{\top}\mathbf{A}_{t-1}^{-1}\mathbf{x}_{t}-\frac{a_{t}^{2}\mathbf{x}_{t}^{\top}\mathbf{A}_{t-1}^{-1}\mathbf{x}_{t}\mathbf{x}_{t}^{\top}\mathbf{A}_{t-1}^{-1}\mathbf{x}_{t}}{\frac{1}{a_{t}}+\mathbf{x}_{t}^{\top}\mathbf{A}_{t-1}^{-1}\mathbf{x}_{t}}+1-a_{t}\\
  &=&  \frac{1+a_{t}\mathbf{x}_{t}^{\top}\mathbf{A}_{t-1}^{-1}\mathbf{x}_{t}-a_{t}}{1+a_{t}\mathbf{x}_{t}^{\top}\mathbf{A}_{t-1}^{-1}\mathbf{x}_{t}}~.
\end{eqnarray*}
\QED
 \end{proof}

%\vspace{-0.5cm}
\section{Proof of \thmref{thm:theorem2}}
\label{proof_theorem2}

\begin{proof}
Using the Woodbury matrix identity we get
\begin{eqnarray}
\mathbf{A}_{t}^{-1} & = & \mathbf{A}_{t-1}^{-1}-\frac{\mathbf{A}_{t-1}^{-1}\mathbf{x}_{t}\mathbf{x}_{t}^{\top}\mathbf{A}_{t-1}^{-1}}{\frac{1}{a_{t}}+\mathbf{x}_{t}^{\top}\mathbf{A}_{t-1}^{-1}\mathbf{x}_{t}}\label{woodbury}~,
\end{eqnarray}
therefore
\begin{align}
\mathbf{A}_{t}^{-1}\mathbf{x}_{t} 
& =  \mathbf{A}_{t-1}^{-1}\mathbf{x}_{t}-\frac{\mathbf{A}_{t-1}^{-1}\mathbf{x}_{t}\mathbf{x}_{t}^{\top}\mathbf{A}_{t-1}^{-1}\mathbf{x}_{t}}{\frac{1}{a_{t}}+\mathbf{x}_{t}^{\top}\mathbf{A}_{t-1}^{-1}\mathbf{x}_{t}}%\nonumber \\
  =  \frac{\mathbf{A}_{t-1}^{-1}\mathbf{x}_{t}}{1+a_{t}\mathbf{x}_{t}^{\top}\mathbf{A}_{t-1}^{-1}\mathbf{x}_{t}}\label{t4}~.
\end{align}
For $t = 1 \dots T$ we have
\begin{eqnarray*}
 &  & \ell_{t}(\texttt{WEMM})+\inf_{\mathbf{u}\in\mathbb{R}^{d}}\left(b\left\Vert \mathbf{u}\right\Vert ^{2}+L_{t-1}^{\boldsymbol{a}}(\mathbf{u})\right)-\inf_{\mathbf{u}\in\mathbb{R}^{d}}\left(b\left\Vert \mathbf{u}\right\Vert ^{2}+L_{t}^{\boldsymbol{a}}(\mathbf{u})\right)\\
 & = &
 \left(y_{t}-\hat{y}_{t}\right)^{2}+\inf_{\mathbf{u}\in\mathbb{R}^{d}}\left(b\left\Vert
     \mathbf{u}\right\Vert
   ^{2}+\sum_{s=1}^{t-1}a_{s}\left(y_{s}-\mathbf{u}^{\top}\mathbf{x}_{s}\right)^{2}\right)-\inf_{\mathbf{u}\in\mathbb{R}^{d}}\left(b\left\Vert
     \mathbf{u}\right\Vert
   ^{2}+\sum_{s=1}^{t}a_{s}\left(y_{s}-\mathbf{u}^{\top}\mathbf{x}_{s}\right)^{2}\right)\\
%\end{eqnarray*}
%
%\begin{eqnarray*}
 & \overset{\eqref{optimal_solution}}{=} & \left(y_{t}-\hat{y}_{t}\right)^{2}+\sum_{s=1}^{t-1}a_{s}y_{s}^{2}-\mathbf{b}_{t-1}^{\top}\mathbf{A}_{t-1}^{-1}\mathbf{b}_{t-1}-\sum_{s=1}^{t}a_{s}y_{s}^{2}+\mathbf{b}_{t}^{\top}\mathbf{A}_{t}^{-1}\mathbf{b}_{t}\\
 & = & \left(y_{t}-\hat{y}_{t}\right)^{2}-a_{t}y_{t}^{2}-\mathbf{b}_{t-1}^{\top}\mathbf{A}_{t-1}^{-1}\mathbf{b}_{t-1}+\mathbf{b}_{t}^{\top}\mathbf{A}_{t}^{-1}\mathbf{b}_{t}\\
 &\overset{\eqref{t3}}{=}  & \left(y_{t}-\hat{y}_{t}\right)^{2}-a_{t}y_{t}^{2}-\mathbf{b}_{t-1}^{\top}\mathbf{A}_{t-1}^{-1}\mathbf{b}_{t-1}+\mathbf{b}_{t-1}^{\top}\mathbf{A}_{t}^{-1}\mathbf{b}_{t-1}+2a_{t}y_{t}\mathbf{b}_{t-1}^{\top}\mathbf{A}_{t}^{-1}\mathbf{x}_{t}+a_{t}^{2}y_{t}^{2}\mathbf{x}_{t}^{\top}\mathbf{A}_{t}^{-1}\mathbf{x}_{t}\\
 & = & \left(y_{t}-\hat{y}_{t}\right)^{2}-a_{t}y_{t}^{2}-\mathbf{b}_{t-1}^{\top}\left(\mathbf{A}_{t-1}^{-1}-\mathbf{A}_{t}^{-1}\right)\mathbf{b}_{t-1}+2a_{t}y_{t}\mathbf{b}_{t-1}^{\top}\mathbf{A}_{t}^{-1}\mathbf{x}_{t}+a_{t}^{2}y_{t}^{2}\mathbf{x}_{t}^{\top}\mathbf{A}_{t}^{-1}\mathbf{x}_{t}\\
 & \overset{\eqref{t2}}{=}  &\left(y_{t}-\hat{y}_{t}\right)^{2}-a_{t}y_{t}^{2}-\mathbf{b}_{t-1}^{\top}\mathbf{A}_{t}^{-1}a_{t}\mathbf{x}_{t}\mathbf{x}_{t}^{\top}\mathbf{A}_{t-1}^{-1}\mathbf{b}_{t-1}+2a_{t}y_{t}\mathbf{b}_{t-1}^{\top}\mathbf{A}_{t}^{-1}\mathbf{x}_{t}+a_{t}^{2}y_{t}^{2}\mathbf{x}_{t}^{\top}\mathbf{A}_{t}^{-1}\mathbf{x}_{t}\\
&=  & \left(y_{t}-\hat{y}_{t}\right)^{2}-a_{t}y_{t}^{2}+a_{t}\left(-\hat{y}_{t}\mathbf{b}_{t-1}^{\top}+2y_{t}\mathbf{b}_{t-1}^{\top}+a_{t}y_{t}^{2}\mathbf{x}_{t}^{\top}\right)\mathbf{A}_{t}^{-1}\mathbf{x}_{t}\\
 &\overset{\eqref{t4}}{=}  & \left(y_{t}-\hat{y}_{t}\right)^{2}-a_{t}y_{t}^{2}+a_{t}\left(-\hat{y}_{t}\mathbf{b}_{t-1}^{\top}+2y_{t}\mathbf{b}_{t-1}^{\top}+a_{t}y_{t}^{2}\mathbf{x}_{t}^{\top}\right)\frac{\mathbf{A}_{t-1}^{-1}\mathbf{x}_{t}}{1+a_{t}\mathbf{x}_{t}^{\top}\mathbf{A}_{t-1}^{-1}\mathbf{x}_{t}}\\
 & = & \left(y_{t}-\hat{y}_{t}\right)^{2}+a_{t}\frac{-y_{t}^{2}-y_{t}^{2}a_{t}\mathbf{x}_{t}^{\top}\mathbf{A}_{t-1}^{-1}\mathbf{x}_{t}-\hat{y}_{t}^{2}+2y_{t}\hat{y}_{t}+a_{t}y_{t}^{2}\mathbf{x}_{t}^{\top}\mathbf{A}_{t-1}^{-1}\mathbf{x}_{t}}{1+a_{t}\mathbf{x}_{t}^{\top}\mathbf{A}_{t-1}^{-1}\mathbf{x}_{t}}\\
 & = & \left(y_{t}-\hat{y}_{t}\right)^{2}-a_{t}\frac{\left(y_{t}-\hat{y}_{t}\right)^{2}}{1+a_{t}\mathbf{x}_{t}^{\top}\mathbf{A}_{t-1}^{-1}\mathbf{x}_{t}}\\
&
=&\frac{1+a_{t}\mathbf{x}_{t}^{\top}\mathbf{A}_{t-1}^{-1}\mathbf{x}_{t}-a_{t}}{1+a_{t}\mathbf{x}_{t}^{\top}\mathbf{A}_{t-1}^{-1}\mathbf{x}_{t}}\left(y_{t}-\hat{y}_{t}\right)^{2}
\leq0 ~.
\end{eqnarray*}
Summing over $t\in\left\{ 1,\ldots,T\right\} $ and using \eqref{LT}
yields
\(
L_{T}(\texttt{WEMM})-\inf_{\mathbf{u}\in\mathbb{R}^{d}}\left(b\left\Vert
    \mathbf{u}\right\Vert
  ^{2}+L_{T}^{\boldsymbol{a}}(\mathbf{u})\right)\leq 0~.
\)
 \QED
\end{proof}

\section{Proof of \thmref{thm:theorem3}}
\label{proof_theorem3}

\begin{proof}
From \eqref{woodbury} we see that $\mathbf{A}_{t}^{-1}\prec\mathbf{A}_{t-1}^{-1}$
and because $\mathbf{A}_{0}=b\mathbf{I}$ we get 
\[
\mathbf{x}_{t}^{\top}\mathbf{A}_{t}^{-1}\mathbf{x}_{t}<\mathbf{x}_{t}^{\top}\mathbf{A}_{t-1}^{-1}\mathbf{x}_{t}<\mathbf{x}_{t}^{\top}\mathbf{A}_{t-2}^{-1}\mathbf{x}_{t}<\ldots<\mathbf{x}_{t}^{\top}\mathbf{A}_{0}^{-1}\mathbf{x}_{t}=\frac{1}{b}\left\Vert \mathbf{x}_{t}\right\Vert ^{2}\leq\frac{1}{b}~,
\]
therefore $1\leq a_{t}\leq\frac{1}{1-\frac{1}{b}}=\frac{b}{b-1}$.
From \eqref{t4} we have
\begin{eqnarray*}
\mathbf{x}_{t}^{\top}\mathbf{A}_{t}^{-1}\mathbf{x}_{t} 
 &=&
 \frac{\mathbf{x}_{t}^{\top}\mathbf{A}_{t-1}^{-1}\mathbf{x}_{t}}{1+a_{t}\mathbf{x}_{t}^{\top}\mathbf{A}_{t-1}^{-1}\mathbf{x}_{t}}\\
&=&\frac{1-\frac{1}{a_{t}}}{1+a_{t}\left(1-\frac{1}{a_{t}}\right)} ~=~\frac{a_{t}-1}{a_{t}^{2}}~,
\end{eqnarray*}
so we can bound the term $a_{t}-1$ as following
\begin{equation}
a_{t}-1=a_{t}^{2}\mathbf{x}_{t}^{\top}\mathbf{A}_{t}^{-1}\mathbf{x}_{t}\leq\frac{b}{b-1}a_{t}\mathbf{x}_{t}^{\top}\mathbf{A}_{t}^{-1}\mathbf{x}_{t}~. \label{t6}
\end{equation}
With an argument similar to~\cite{Forster} we have,
 \(
 a_{t}\mathbf{x}_{t}^{\top}\mathbf{A}_{t}^{-1}\mathbf{x}_{t}
%=\mathbf{\xi}^{\top}\mathbf{\xi}
%\leq\ln\frac{1}{\left|\mathbf{I}-\mathbf{\xi}\mathbf{\xi}^{\top}\right|}
%=\ln\frac{\left|\mathbf{A}\mathbf{A}\right|}{\left|\mathbf{A}\mathbf{A}-\mathbf{A}\mathbf{\xi}\mathbf{\xi}^{\top}\mathbf{A}\right|}
\leq\ln\frac{\left|\mathbf{A}_{t}\right|}{\left|\mathbf{A}_{t}-a_{t}\mathbf{x}_{t}\mathbf{x}_{t}^{\top}\right|}
=\ln\frac{\left|\mathbf{A}_{t}\right|}{\left|\mathbf{A}_{t-1}\right|} ~.
 \)
Summing the last inequality over $t$ and using the initial value
$\ln\left|\frac{1}{b}\mathbf{A}_{0}\right|=0$ we get
\begin{equation}
\sum_{t=1}^{T}a_{t}\mathbf{x}_{t}^{\top}\mathbf{A}_{t}^{-1}\mathbf{x}_{t}\leq\ln\left|\frac{1}{b}\mathbf{A}_{T}\right|~.  \label{t7}
\end{equation}
Substituting the last equation in 
%\eqref{t7} in 
\eqref{t6} we get the logarithmic bound 
\(
\sum_{t=1}^{T}\left(a_{t}-1\right)\leq\frac{b}{b-1}\ln\left|\frac{1}{b}\mathbf{A}_{T}\right| ~,
\)
as required.
% \begin{eqnarray*}
% L_{T}^{\boldsymbol{a}}(\mathbf{u}) & = & L_{T}(\mathbf{u})+\sum_{t=1}^{T}\left(a_{t}-1\right)\ell_{t}(\mathbf{u})\leq L_{T}(\mathbf{u})+\frac{b}{b-1}S\ln\left|\frac{1}{b}\mathbf{A}_{T}\right|
% \end{eqnarray*}
\QED
\end{proof}

\section{Proof of \lemref{lem:lemma4}}
\label{proof_lemma4}

\begin{proof}
We set the derivative of $J$ with respect to $\vui{t}$ to zero,
\[
\frac{\partial J}{\partial\mathbf{u}_{t}}=2c\left(\mathbf{u}_{t}-\bar{\mathbf{u}}\right)-2\widetilde{a}_{t}\left(y_{t}-\mathbf{u}_{t}^{\top}\mathbf{x}_{t}\right)\mathbf{x}_{t}=0~,
\]
and solve for $\vui{t}$: 
\begin{eqnarray}
\mathbf{u}_{t} 
&=&  \left(c\mathbf{I}+\widetilde{a}_{t}\mathbf{x}_{t}\mathbf{x}_{t}^{\top}\right)^{-1}\left(c\bar{\mathbf{u}}+\widetilde{a}_{t}y_{t}\mathbf{x}_{t}\right)\nonumber \\
 & = & \left(c^{-1}\mathbf{I}-\frac{c^{-2}}{\widetilde{a}_{t}^{-1}+c^{-1}\left\Vert \mathbf{x}_{t}\right\Vert ^{2}}\mathbf{x}_{t}\mathbf{x}_{t}^{\top}\right)\left(c\bar{\mathbf{u}}+\widetilde{a}_{t}y_{t}\mathbf{x}_{t}\right)\nonumber \\
 &=&
 \bar{\mathbf{u}}+\frac{c^{-1}}{\widetilde{a}_{t}^{-1}+c^{-1}\left\Vert
     \mathbf{x}_{t}\right\Vert
   ^{2}}\left(y_{t}-\bar{\mathbf{u}}^{\top}\mathbf{x}_{t}\right)\mathbf{x}_{t} \label{uopt}~.
\end{eqnarray}
For the optimal $\mathbf{u}_{t}$ of \eqref{uopt}, we compute the
following two terms, which are used next,
%the terms $\left(y_{t}-\mathbf{u}_{t}^{\top}\mathbf{x}_{t}\right)^{2}$
%and $\left\Vert \mathbf{u}_{t}-\bar{\mathbf{u}}\right\Vert ^{2}$:
\begin{align}
\left(y_{t}-\mathbf{u}_{t}^{\top}\mathbf{x}_{t}\right)^{2} 
& ~=  \left(y_{t}-\left(\bar{\mathbf{u}}+\frac{c^{-1}}{\widetilde{a}_{t}^{-1}+c^{-1}\left\Vert \mathbf{x}_{t}\right\Vert ^{2}}\left(y_{t}-\bar{\mathbf{u}}^{\top}\mathbf{x}_{t}\right)\mathbf{x}_{t}\right)^{\top}\mathbf{x}_{t}\right)^{2}\nonumber \\
 &~ =  \frac{\widetilde{a}_{t}^{-2}}{\left(\widetilde{a}_{t}^{-1}+c^{-1}\left\Vert \mathbf{x}_{t}\right\Vert ^{2}\right)^{2}}\left(y_{t}-\bar{\mathbf{u}}^{\top}\mathbf{x}_{t}\right)^{2}\label{t8}\\
\left\Vert \mathbf{u}_{t}-\bar{\mathbf{u}}\right\Vert ^{2} 
& ~=  \left\Vert \frac{c^{-1}}{\widetilde{a}_{t}^{-1}+c^{-1}\left\Vert \mathbf{x}_{t}\right\Vert ^{2}}\left(y_{t}-\bar{\mathbf{u}}^{\top}\mathbf{x}_{t}\right)\mathbf{x}_{t}\right\Vert ^{2}\nonumber \\
 &~ =  \frac{c^{-2}}{\left(\widetilde{a}_{t}^{-1}+c^{-1}\left\Vert \mathbf{x}_{t}\right\Vert ^{2}\right)^{2}}\left(y_{t}-\bar{\mathbf{u}}^{\top}\mathbf{x}_{t}\right)^{2}\left\Vert \mathbf{x}_{t}\right\Vert ^{2}\label{t9}~.
\end{align}
From \eqref{t8} and \eqref{t9} we get
\[
c\left\Vert \mathbf{u}_{t}-\bar{\mathbf{u}}\right\Vert ^{2}+\widetilde{a}_{t}\left(y_{t}-\mathbf{u}_{t}^{\top}\mathbf{x}_{t}\right)^{2}=\frac{1}{\widetilde{a}_{t}^{-1}+c^{-1}\left\Vert \mathbf{x}_{t}\right\Vert ^{2}}\left(y_{t}-\bar{\mathbf{u}}^{\top}\mathbf{x}_{t}\right)^{2}~.
\]
%which after substituting  in
%$J\left(\mathbf{u}_{1},\ldots,\mathbf{u}_{T}\right)$ completes the proof.
 Therefore the minimal value of $J\left(\mathbf{u}_{1},\ldots,\mathbf{u}_{T}\right)$
 is given by,
 \begin{eqnarray*}
 J_{min} & = & b\left\Vert \bar{\mathbf{u}}\right\Vert ^{2}+\sum_{t=1}^{T}\frac{1}{\widetilde{a}_{t}^{-1}+c^{-1}\left\Vert \mathbf{x}_{t}\right\Vert ^{2}}\left(y_{t}-\bar{\mathbf{u}}^{\top}\mathbf{x}_{t}\right)^{2}~,
 \end{eqnarray*}
which completes the proof.
\QED
 \end{proof}

\bibliographystyle{plain}
\bibliography{bib}

\begin{thebibliography}{10}

\bibitem{AzouryWa01}
Katy~S. Azoury and Manfred~K. Warmuth.
\newblock Relative loss bounds for on-line density estimation with the
  exponential family of distributions.
\newblock {\em Machine Learning}, 43(3):211--246, 2001.

\bibitem{Bershad}
Neil~J. Bershad.
\newblock Analysis of the normalized lms algorithm with gaussian inputs.
\newblock {\em IEEE Transactions on Acoustics, Speech, and Signal Processing},
  34(4):793--806, 1986.

\bibitem{Bitmead}
Robert~R. Bitmead and Brian D.~O. Anderson.
\newblock Performance of adaptive estimation algorithms in dependent random
  environments.
\newblock {\em IEEE Transactions on Automatic Control}, 25:788--794, 1980.

\bibitem{BusuttilK07}
Steven Busuttil and Yuri Kalnishkan.
\newblock Online regression competitive with changing predictors.
\newblock In {\em ALT}, pages 181--195, 2007.

\bibitem{Nicolo_Warmuth}
Nicolo Ceas-Bianchi, Philip~M. Long, and Manfred~K. Warmuth.
\newblock Worst case quadratic loss bounds for on-line prediction of linear
  functions by gradient descent.
\newblock Technical Report IR-418, University of California, Santa Cruz, CA,
  USA, 1993.

\bibitem{CesaBianchiCoGe05}
Nicolo Cesa-Bianchi, Alex Conconi, and Claudio Gentile.
\newblock A second-order perceptron algorithm.
\newblock {\em Siam Journal of Commutation}, 34(3):640--668, 2005.

\bibitem{CesaBiGa06}
Nicolo Cesa-Bianchi and Gabor Lugosi.
\newblock {\em Prediction, Learning, and Games}.
\newblock Cambridge University Press, New York, NY, USA, 2006.

\bibitem{CrammerKuDr09}
Koby Crammer, Alex Kulesza, and Mark Dredze.
\newblock Adaptive regularization of weighted vectors.
\newblock In {\em Advances in Neural Information Processing Systems 23}, 2009.

\bibitem{CrammerKuDr12}
Koby Crammer, Alex Kulesza, and Mark Dredze.
\newblock New $\mathcal{H}\infty$ bounds for the recursive least squares
  algorithm exploiting input structure.
\newblock In {\em ICASSP}, pages 2017--2020, 2012.

\bibitem{DekelLoSi07}
Ofer Dekel, Philip~M. Long, and Yoram Singer.
\newblock Online learning of multiple tasks with a shared loss.
\newblock {\em Journal of Machine Learning Research}, 8:2233--2264, 2007.

\bibitem{DuchiHS10}
John Duchi, Elad Hazan, and Yoram Singer.
\newblock Adaptive subgradient methods for online learning and stochastic
  optimization.
\newblock In {\em COLT}, pages 257--269, 2010.

\bibitem{Forster}
Jurgen Forster.
\newblock On relative loss bounds in generalized linear regression.
\newblock In {\em FCT}, 1999.

\bibitem{Foster91}
Dean~P. Foster.
\newblock Prediction in the worst case.
\newblock {\em The An. of Stat.}, 19(2):1084--1090, 1991.

\bibitem{Golub:1996:MC:248979}
Gene~H. Golub and Charles~F. Van~Loan.
\newblock {\em Matrix computations (3rd ed.)}.
\newblock Johns Hopkins University Press, Baltimore, MD, USA, 1996.

\bibitem{Hayes}
Monson~H. Hayes.
\newblock 9.4: Recursive least squares.
\newblock In {\em Statistical Digital Signal Processing and Modeling}, page
  541, 1996.

\bibitem{HazanK09a}
Elad Hazan and Satyen Kale.
\newblock On stochastic and worst-case models for investing.
\newblock In {\em NIPS}, pages 709--717, 2009.

\bibitem{HerbsterW01}
Mark Herbster and Manfred~K. Warmuth.
\newblock Tracking the best linear predictor.
\newblock {\em Journal of Machine Learning Research}, 1:281--309, 2001.

\bibitem{Kalman60}
Rudolph~Emil Kalman.
\newblock A new approach to linear filtering and prediction problems.
\newblock {\em Transactions of the ASME--Journal of Basic Engineering},
  82(Series D):35--45, 1960.

\bibitem{Kiv_War}
Jyrki Kivinen and Manfred~K. Warmuth.
\newblock Exponential gradient versus gradient descent for linear predictors.
\newblock {\em Information and Computation}, 132:132--163, 1997.

\bibitem{McMahanS10}
Hugh~Brendan McMahan and Matthew~J. Streeter.
\newblock Adaptive bound optimization for online convex optimization.
\newblock In {\em COLT}, pages 244--256, 2010.

\bibitem{OrabonaCBG12}
Francesco Orabona, Nicolo Cesa-Bianchi, and Claudio Gentile.
\newblock Beyond logarithmic bounds in online learning.
\newblock In {\em AISTATS}, 2012.
\newblock to appear.

\bibitem{DBLP:journals/tsp/Simon06}
Dan Simon.
\newblock A game theory approach to constrained minimax state estimation.
\newblock {\em IEEE Transactions on Signal Processing}, 54(2):405--412, 2006.

\bibitem{Simon:2006:OSE:1146304}
Dan Simon.
\newblock {\em Optimal State Estimation: Kalman, H Infinity, and Nonlinear
  Approaches}.
\newblock Wiley-Interscience, 2006.

\bibitem{TakimotoW00}
Eiji Takimoto and Manfred~K. Warmuth.
\newblock The last-step minimax algorithm.
\newblock In {\em ALT}, 2000.

\bibitem{VaitsCr11}
Nina Vaits and Koby Crammer.
\newblock Re-adapting the regularization of weights for non-stationary
  regression.
\newblock In {\em ALT}, 2011.

\bibitem{vovkAS}
Volodimir~G. Vovk.
\newblock Aggregating strategies.
\newblock In {\em COLT}, 1990.

\bibitem{Vovk01}
Volodya Vovk.
\newblock Competitive on-line statistics.
\newblock {\em International Statistical Review}, 69, 2001.

\bibitem{WidrowHoff}
Bernard Widrow and Jr. Marcian E.~Hoff.
\newblock Adaptive switching circuits.
\newblock 1960.

\bibitem{Zinkevich03}
Martin Zinkevich.
\newblock Online convex programming and generalized infinitesimal gradient
  ascent.
\newblock In {\em ICML}, 2003.

\end{thebibliography}

\end{document}